\documentclass[10pt,twocolumn,letterpaper]{article}

\usepackage[pagenumbers]{cvpr} % To force page numbers, e.g. for an arXiv version

\usepackage[utf8]{inputenc}
\usepackage[T1]{fontenc}
\usepackage[dvipsnames]{xcolor}
\usepackage{amsfonts}
\usepackage{amsthm}
\usepackage{bm}
\usepackage{mathtools}
\usepackage{algorithm}
\usepackage{algpseudocode}
\usepackage{adjustbox}
\usepackage{makecell}
\usepackage{multirow}
\usepackage{cuted}
\usepackage{capt-of}
\usepackage{listings}

\newcommand{\E}{\mathbb{E}}
\newcommand{\rvx}{\mathbf{x}}
\newcommand{\rvc}{\mathbf{c}}

\newtheorem{lemma}{Lemma}[section]

\newcommand{\Stable}{{\color{ForestGreen}{Stable}}}
\newcommand{\Unstable}{{\color{BrickRed}{Unstable}}}

\definecolor{codegreen}{rgb}{0,0.6,0}
\definecolor{codegray}{rgb}{0.5,0.5,0.5}
\definecolor{codepurple}{rgb}{0.58,0,0.82}
\definecolor{backcolor}{rgb}{0.95,0.95,0.92}

\lstdefinestyle{codestyle}{
    sensitive=true,
    morekeywords={import,as,def,None,return},
    keywordstyle=\color{codepurple},
    morecomment=[l]{\#},
    morecomment=[s]{"""}{"""},
    commentstyle=\color{codegreen},
    emph={[1]jax,numpy,jnp,Array,float},
    emphstyle={[1]\color{RoyalBlue}},
    emph={[2]prdp_loss,clip,mean,maximum,float},
    emphstyle={[2]\color{Maroon}},
    backgroundcolor=\color{backcolor},
    numberstyle=\tiny\color{codegray},
    basicstyle=\ttfamily\footnotesize,
    breakatwhitespace=false,
    breaklines=true,
    captionpos=b,
    keepspaces=true,
    numbers=left,
    numbersep=5pt,
    showspaces=false,
    showstringspaces=false,
    showtabs=false,
    tabsize=2,
    framexleftmargin=3pt
}

\definecolor{cvprblue}{rgb}{0.21,0.49,0.74}
\usepackage[pagebackref,breaklinks,colorlinks,citecolor=cvprblue]{hyperref}

\def\paperID{5416} % *** Enter the Paper ID here
\def\confName{CVPR}
\def\confYear{2024}

\title{PRDP: Proximal Reward Difference Prediction\\for Large-Scale Reward Finetuning of Diffusion Models}

\author{Fei Deng$^{1,2}$\thanks{Work done during an internship at Google.},\: Qifei Wang$^{1}$,\, Wei Wei$^{3}$\thanks{Work done while working at Google.},\: Matthias Grundmann$^{1}$,\, Tingbo Hou$^{1}$\\
{\small $^1$Google,\: $^2$Rutgers University,\: $^3$Accenture}\\
{\small \,\, \url{https://fdeng18.github.io/prdp}}
}

\begin{document}
\maketitle

\begin{strip}
    \centering
    \vskip -0.5in
    \includegraphics[width=\textwidth]{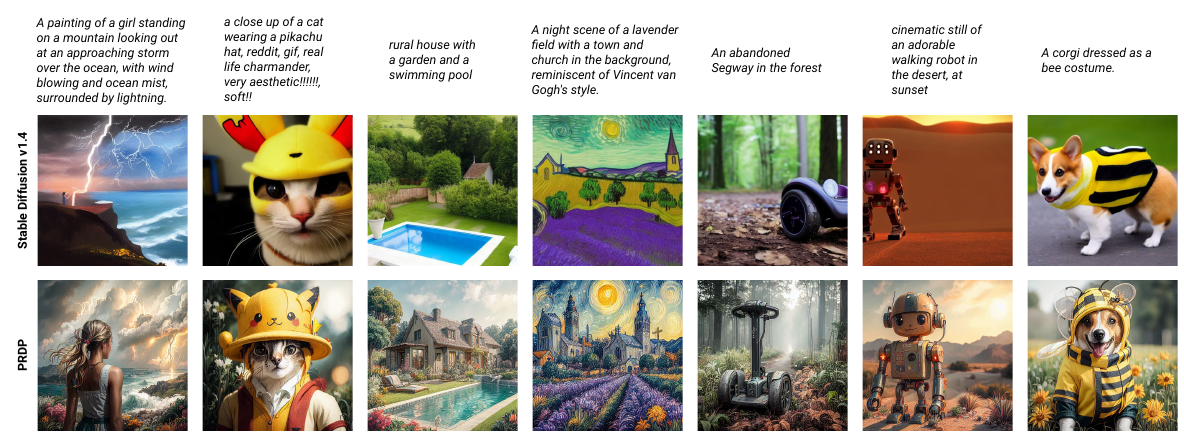}
    \captionof{figure}{\textbf{Generation samples on complex, unseen prompts.} Our proposed method, PRDP, achieves stable black-box reward finetuning for diffusion models for the first time on large-scale prompt datasets, leading to superior generation quality on complex, unseen prompts. Here, PRDP is finetuned from Stable Diffusion v1.4 on the training set prompts of Pick-a-Pic v1 dataset, using a weighted combination of rewards: PickScore $= 10$, HPSv2 $= 2$, Aesthetic $= 0.05$. The images within each column are generated using the same random seed. \label{fig:teaser}}
    \vskip 0.1in
\end{strip}

\begin{abstract}
Reward finetuning has emerged as a promising approach to aligning foundation models with downstream objectives. Remarkable success has been achieved in the language domain by using reinforcement learning (RL) to maximize rewards that reflect human preference. However, in the vision domain, existing RL-based reward finetuning methods are limited by their instability in large-scale training, rendering them incapable of generalizing to complex, unseen prompts. In this paper, we propose Proximal Reward Difference Prediction (PRDP), enabling stable black-box reward finetuning for diffusion models for the first time on large-scale prompt datasets with over 100K prompts. Our key innovation is the Reward Difference Prediction (RDP) objective that has the same optimal solution as the RL objective while enjoying better training stability. Specifically, the RDP objective is a supervised regression objective that tasks the diffusion model with predicting the reward difference of generated image pairs from their denoising trajectories. We theoretically prove that the diffusion model that obtains perfect reward difference prediction is exactly the maximizer of the RL objective. We further develop an online algorithm with proximal updates to stably optimize the RDP objective. In experiments, we demonstrate that PRDP can match the reward maximization ability of well-established RL-based methods in small-scale training. Furthermore, through large-scale training on text prompts from the Human Preference Dataset v2 and the Pick-a-Pic v1 dataset, PRDP achieves superior generation quality on a diverse set of complex, unseen prompts whereas RL-based methods completely fail.
\end{abstract}    
\section{Introduction}

Diffusion models have achieved remarkable success in generative modeling of continuous data, especially in photorealistic text-to-image synthesis~\citep{sohl2015deep,ddpm,song2021scorebased,adm,stable-diffusion,dalle2,imagen,glide}. However, the maximum likelihood training objective of diffusion models is often misaligned with their downstream use cases, such as generating novel compositions of objects unseen during training, and producing images that are aesthetically preferred by humans.

A similar misalignment problem exists in language models, where exactly matching the model output to the training distribution tends to yield undesirable model behavior. For example, the model may output biased, toxic, or harmful content. A successful solution, called reinforcement learning from human feedback (RLHF)~\citep{ziegler2019fine,stiennon2020learning,InstructGPT,bai2022training}, is to use reinforcement learning (RL) to finetune the language model such that it maximizes some reward function that reflects human preference. Typically, the reward function is defined by a reward model pretrained from human preference data.

Inspired by the success of RLHF in language models, researchers have developed several reward models in the vision domain~\citep{image-reward,wu2023human,HPSv2,PickScore,lee2023aligning} that are similarly trained to be aligned with human preference. Furthermore, two recent works, DDPO~\citep{ddpo} and DPOK~\citep{dpok}, have explored using RL to finetune diffusion models. They both view the denoising process as a Markov decision process~\citep{fan2023optimizing}, and apply policy gradient methods such as PPO~\citep{schulman2017ppo} to maximize rewards.

However, policy gradients are notoriously prone to high variance, causing training instability. To reduce variance, a common approach is to normalize the rewards by subtracting their expected value~\citep{williams1992simple,sutton2018reinforcement}. DPOK fits a value function to estimate the expected reward, showing promising results when trained on ${\sim} 200$ prompts. Alternatively, DDPO maintains a separate buffer for each prompt to track the mean and variance of rewards, demonstrating stable training on ${\sim} 400$ prompts and better performance than DPOK. Nevertheless, we find that DDPO still suffers from training instability on larger numbers of prompts, depriving it of the benefits offered by training on large-scale prompt datasets.

In this paper, we propose Proximal Reward Difference Prediction (PRDP), a scalable reward maximization algorithm that does not rely on policy gradients. To the best of our knowledge, PRDP is the first method that achieves stable large-scale finetuning of diffusion models on more than $100$K prompts for black-box reward functions.

Inspired by the recent success of DPO~\citep{rafailov2023dpo} that converts the RLHF objective for language models into a supervised classification objective, we derive for diffusion models a new supervised regression objective, called Reward Difference Prediction (RDP), that has the same optimal solution as the RLHF objective while enjoying better training stability. Specifically, our RDP objective tasks the diffusion model with predicting the reward difference of generated image pairs from their denoising trajectories. We prove that the diffusion model that obtains perfect reward difference prediction is exactly the maximizer of the RLHF objective. We further propose proximal updates and online optimization to improve training stability and generation quality.

Our contributions are summarized as follows:
\begin{itemize}
    \item We propose PRDP, a scalable reward finetuning method for diffusion models, with a new reward difference prediction objective and its stable optimization algorithm.
    \item PRDP achieves stable black-box reward maximization for diffusion models for the first time on large-scale prompt datasets with over $100$K prompts.
    \item PRDP exhibits superior generation quality and generalization to unseen prompts through large-scale training.
\end{itemize}

\section{Preliminaries}

In this section, we briefly introduce the generative process of denoising diffusion probabilistic models (DDPMs)~\citep{sohl2015deep,ddpm,song2021scorebased}. Given a text prompt $\rvc$, a text-to-image DDPM $\pi_\theta$ with parameters $\theta$ defines a text-conditioned image distribution $\pi_\theta(\rvx_0 | \rvc)$ as follows:
\begin{align}
\label{eqn:diffusion}
\begin{split}
    \pi_\theta(\rvx_0 | \rvc) &= \int \pi_\theta(\rvx_{0:T} | \rvc) \, \mathrm{d}\rvx_{1:T} \\
    &= \int p(\rvx_T) \prod_{t=1}^T \pi_\theta(\rvx_{t-1} | \rvx_t, \rvc) \, \mathrm{d}\rvx_{1:T},
\end{split}
\end{align}
where $\rvx_0$ is the image, and $\rvx_{1:T}$ are latent variables of the same dimension as $\rvx_0$. Typically, $p(\rvx_T) = \mathcal{N}(\mathbf{0}, \mathbf{I})$, and
\begin{align}
    \pi_\theta(\rvx_{t-1} | \rvx_t, \rvc) = \mathcal{N}(\rvx_{t-1}; \bm{\mu}_\theta(\rvx_t, \rvc), \sigma_t^2\mathbf{I})
\end{align}
is a Gaussian distribution with learnable mean and fixed covariance. To generate an image $\rvx_0 \sim \pi_\theta(\rvx_0 | \rvc)$, DDPM uses ancestral sampling. That is, it samples the full denoising trajectory $\rvx_{0:T} \sim \pi_\theta(\rvx_{0:T} | \rvc)$, by first sampling $\rvx_T \sim p(\rvx_T)$, and then sampling $\rvx_{t-1} \sim \pi_\theta(\rvx_{t-1} | \rvx_t, \rvc)$ for $t = T, \dots, 1$. Conversely, given a denoising trajectory $\rvx_{0:T}$, we can analytically compute its log-likelihood as
\begin{align}
    \label{eqn:log_likelihood}
    \log \pi_\theta(\rvx_{0:T} | \rvc) &= \log p(\rvx_T) + \sum_{t=1}^T \log \pi_\theta(\rvx_{t-1} | \rvx_t, \rvc) \\
    &= -\frac{1}{2}\sum_{t=1}^T \frac{\lVert \rvx_{t-1} - \bm{\mu}_\theta(\rvx_t, \rvc) \rVert^2}{\sigma_t^2} + C,
\end{align}
where $C$ is a constant independent of $\theta$.

\section{Method}

\begin{figure*}[t]
    \centering
    \includegraphics[width=\textwidth]{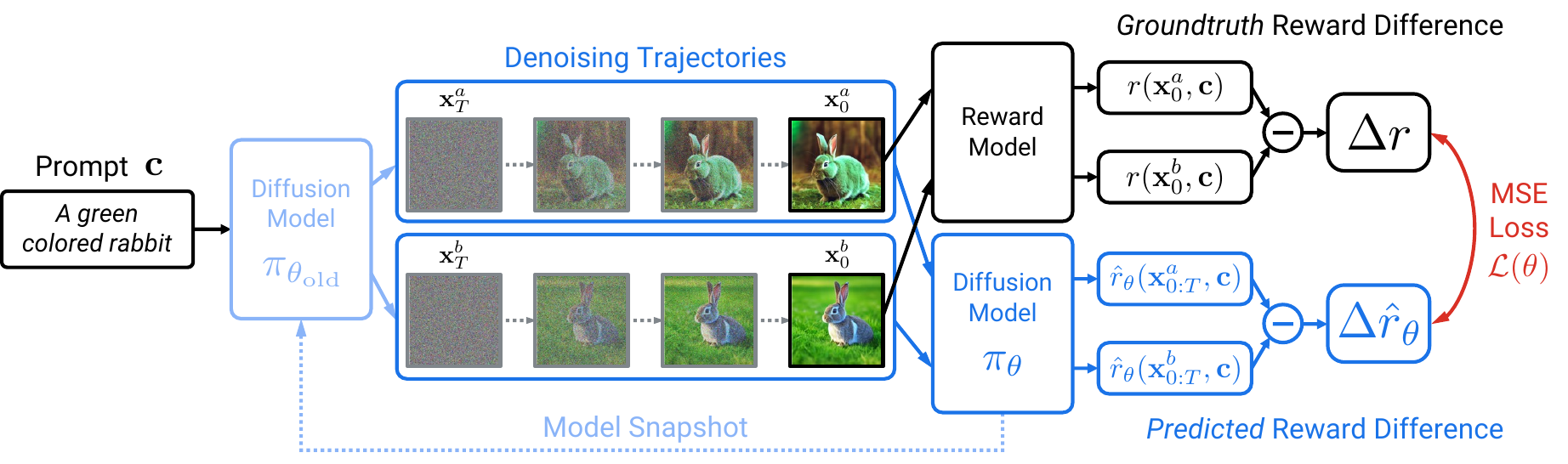}
    \caption{\textbf{PRDP framework.} PRDP mitigates the instability of policy gradient methods by converting the RLHF objective to an equivalent supervised regression objective. Specifically, given a text prompt, PRDP samples two images, and tasks the diffusion model with predicting the reward difference of these two images from their denoising trajectories. The diffusion model is updated by stochastic gradient descent on the MSE loss that measures the prediction error. We prove that the MSE loss and the RLHF objective have the same optimal solution.}
    \label{fig:prdp}
\end{figure*}

\subsection{Reward Difference Prediction for\\KL-Regularized Reward Maximization}

We start derivation from the typical RLHF objective~\citep{dpok}:
\begin{align}
\label{eqn:rlhf_objective}
    \max_{\pi_\theta} \ \E_{\rvx_0, \rvc} \! \left[ r(\rvx_0, \rvc) - \beta \mathrm{KL}[\pi_\theta(\rvx_0 | \rvc) || \pi_\mathrm{ref}(\rvx_0 | \rvc)] \right].
\end{align}
Here, we seek to finetune the diffusion model $\pi_\theta$ by maximizing a given reward function $r(\rvx_0, \rvc)$ with a KL regularization, whose strength is controlled by a hyperparameter $\beta$. The reward function can be a pretrained reward model (\eg, HPSv2~\citep{HPSv2}, PickScore~\citep{PickScore}) that measures the generation quality, and the KL regularization discourages $\pi_\theta$ from deviating too far from the pretrained diffusion model $\pi_\mathrm{ref}$ (\eg, Stable Diffusion~\citep{stable-diffusion}). This helps $\pi_\theta$ to preserve the overall generation capability of $\pi_\mathrm{ref}$, and keeps the generated images $\rvx_0$ close to the distribution where the reward model is accurate. The expectation is taken over text prompts $\rvc \sim p(\rvc)$ and images $\rvx_0 \sim \pi_\theta(\rvx_0 | \rvc)$, where $p(\rvc)$ is a predefined prompt distribution, usually a uniform distribution over a set of training prompts.

In contrast to language models, the KL regularization in \cref{eqn:rlhf_objective} cannot be computed analytically, due to the intractable integral defined in \cref{eqn:diffusion}. Hence, we instead maximize a lower bound of the objective in \cref{eqn:rlhf_objective}:
\begin{align}
\label{eqn:objective}
    \max_{\pi_\theta} \ \E_{\rvx_0, \rvc} \! \left[ r(\rvx_0, \rvc) - \beta \mathrm{KL}[\pi_\theta(\bar{\rvx} | \rvc) || \pi_\mathrm{ref}(\bar{\rvx} | \rvc)] \right],
\end{align}
where $\bar{\rvx} \coloneqq \rvx_{0:T}$ is the full denoising trajectory. We provide the proof of lower bound in \cref{sec:lower_bound}.

While it is possible to apply REINFORCE~\citep{williams1992simple} or more advanced policy gradient methods~\citep{schulman2017ppo,ddpo,dpok} to optimize \cref{eqn:objective}, we empirically find they are hard to scale to large numbers of prompts due to training instability. Inspired by DPO~\citep{rafailov2023dpo}, we propose to reformulate \cref{eqn:objective} into a supervised learning objective, allowing stable training on more than $100$K prompts.

First, we derive the optimal solution to \cref{eqn:objective} as:
\begin{align}
\label{eqn:optimal_solution}
    \pi_{\theta^\star}(\bar{\rvx} | \rvc) = \frac{1}{Z(\rvc)} \pi_\mathrm{ref}(\bar{\rvx} | \rvc) \exp \! \left( \frac{1}{\beta} r(\rvx_0, \rvc) \right),
\end{align}
where $Z(\rvc) \! = \! \int \! \pi_\mathrm{ref}(\bar{\rvx} | \rvc) \! \exp \! \left( r(\rvx_0, \rvc) / \beta \right) \! \mathrm{d}\bar{\rvx}$ is the partition function. Proof can be found in \cref{sec:maximizer}. Since $Z(\rvc)$ is intractable, \cref{eqn:optimal_solution} cannot be directly used to compute $\pi_{\theta^\star}$. However, it reveals that $\pi_{\theta^\star}$ must satisfy
\begin{align}
    \log \frac{\pi_{\theta^\star}(\bar{\rvx} | \rvc)}{\pi_\mathrm{ref}(\bar{\rvx} | \rvc)} = \frac{1}{\beta} r(\rvx_0, \rvc) - \log Z(\rvc)
\end{align}
for all $\bar{\rvx}$ and $\rvc$. This allows us to cancel the $\log Z(\rvc)$ term by considering two denoising trajectories $\bar{\rvx}^a$ and $\bar{\rvx}^b$ that correspond to the same text prompt $\rvc$:
\begin{align}
\label{eqn:reward_diff}
    \log \frac{\pi_{\theta^\star}(\bar{\rvx}^a | \rvc)}{\pi_\mathrm{ref}(\bar{\rvx}^a | \rvc)} - \log \frac{\pi_{\theta^\star}(\bar{\rvx}^b | \rvc)}{\pi_\mathrm{ref}(\bar{\rvx}^b | \rvc)} = \frac{r(\rvx_0^a, \rvc) - r(\rvx_0^b, \rvc)}{\beta}.
\end{align}
Define
\begin{align}
    \label{eqn:r_theta}
    \hat{r}_\theta(\bar{\rvx}, \rvc) &\coloneqq \log \frac{\pi_\theta(\bar{\rvx} | \rvc)}{\pi_\mathrm{ref}(\bar{\rvx} | \rvc)}, \\
    \Delta \hat{r}_\theta(\bar{\rvx}^a, \bar{\rvx}^b, \rvc) &\coloneqq \hat{r}_\theta(\bar{\rvx}^a, \rvc) - \hat{r}_\theta(\bar{\rvx}^b, \rvc), \\
    \Delta r(\rvx_0^a, \rvx_0^b, \rvc) &\coloneqq r(\rvx_0^a, \rvc) - r(\rvx_0^b, \rvc),
\end{align}
then \cref{eqn:reward_diff} becomes
\begin{align}
    \Delta \hat{r}_{\theta^\star}(\bar{\rvx}^a, \bar{\rvx}^b, \rvc) = \Delta r(\rvx_0^a, \rvx_0^b, \rvc) / \beta.
\end{align}
This motivates us to optimize $\pi_\theta$ by minimizing the following mean squared error (MSE) loss:
\begin{align}
\label{eqn:mse_loss}
    \mathcal{L}(\theta) ={}& \E_{\bar{\rvx}^a, \bar{\rvx}^b, \rvc} \ [l_\theta(\bar{\rvx}^a, \bar{\rvx}^b, \rvc)] \\ \notag
    \coloneqq{}& \E_{\bar{\rvx}^a, \bar{\rvx}^b, \rvc} \left\lVert \Delta \hat{r}_\theta(\bar{\rvx}^a, \bar{\rvx}^b, \rvc) - \Delta r(\rvx_0^a, \rvx_0^b, \rvc) / \beta \right\rVert^2.
\end{align}
We call $\mathcal{L}(\theta)$ the Reward Difference Prediction (RDP) objective, since we learn $\pi_\theta$ by predicting the reward difference $\Delta r(\rvx_0^a, \rvx_0^b, \rvc)$ instead of directly maximizing the reward. An illustration is provided in \cref{fig:prdp}. We further show in \cref{sec:equivalent_condition} that
\begin{align}
    \pi_\theta = \pi_{\theta^\star} \iff \mathcal{L}(\theta) = 0.
\end{align}

\subsection{Online Optimization}

\begin{figure*}[t]
    \centering
    \includegraphics[width=0.97\textwidth]{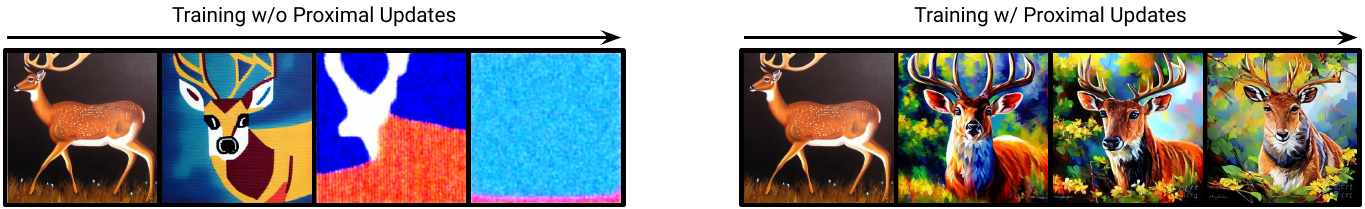}
    \caption{\textbf{Effect of proximal updates.} We show generation samples during the PRDP training process. Here, we use the small-scale setup described in \cref{sec:exp_setup} and HPSv2 as the reward model. All samples use the same prompt ``A painting of a deer'' and the same random seed. (Left) Without proximal updates, training is quite unstable, and the generation quickly becomes meaningless noise. (Right) With proximal updates, the training stability is remarkably improved.}
    \label{fig:ablation_clip}
\end{figure*}

\begin{algorithm}[t]
\small
\caption{PRDP Training}
\label{alg:training}
\begin{algorithmic}[1]
    \Require pretrained diffusion model $\pi_\mathrm{ref}$, training prompt distribution $p(\rvc)$, reward model $r(\rvx_0, \rvc)$, training epochs $E$, gradient updates $K$ per epoch, prompt batch size $N$, image batch size $B$ per prompt
    \State $\pi_\theta \gets \pi_\mathrm{ref}$
    {\color{gray} \Comment{Initialization}}
    \For{epoch $e = 1, \dots, E$}
        \State $\pi_{\theta_\mathrm{old}} \gets \pi_\theta$
        {\color{gray} \Comment{Model snapshot}}
        \State $\{\rvc^n\}_{n=1}^N \stackrel{iid}{\sim} p(\rvc)$
        {\color{gray} \Comment{Sample text prompts}}
        \For{\textbf{each} text prompt $\rvc^n$}
            \State $\{\bar{\rvx}^{n,i}\}_{i=1}^B \stackrel{iid}{\sim} \pi_{\theta_\mathrm{old}}(\bar{\rvx} | \rvc^n)$
            {\color{gray} \Comment{Denoising trajectories}}
        \EndFor
        \State Obtain rewards $r(\rvx_0^{n,i}, \rvc^n)$ for all $n, i$
        \For{gradient step $k = 1, \dots, K$}
            \State $\mathcal{L}(\theta) \gets \frac{1}{N\binom{B}{2}} \sum_{n=1}^N \sum_{1 \leq i < j \leq B} l_\theta(\bar{\rvx}^{n,i}, \bar{\rvx}^{n,j}, \rvc^n)$
            \State Update model parameters $\theta$ by gradient descent
        \EndFor
    \EndFor
\end{algorithmic}
\end{algorithm}

To estimate the expectation in $\mathcal{L}(\theta)$, we need samples of denoising trajectories $\bar{\rvx}^a$ and $\bar{\rvx}^b$ that correspond to the same prompt $\rvc$. A straightforward approach, as similarly done in DPO, is to sample $\bar{\rvx}^a, \bar{\rvx}^b \stackrel{iid}{\sim} \pi_\mathrm{ref}(\bar{\rvx} | \rvc)$. This can be implemented as uniform sampling from a fixed offline dataset generated by the pretrained model $\pi_\mathrm{ref}$.

However, the offline dataset lacks sufficient coverage of samples from $\pi_\theta(\bar{\rvx} | \rvc)$ that keeps updating, leading to suboptimal generation quality. Therefore, we propose an online optimization procedure, inspired by online RL algorithms. Specifically, we sample $\bar{\rvx}^a, \bar{\rvx}^b \stackrel{iid}{\sim} \pi_{\theta_\mathrm{old}}(\bar{\rvx} | \rvc)$, where $\theta_\mathrm{old}$ is a snapshot of the diffusion model parameters $\theta$, and we set $\theta_\mathrm{old} \gets \theta$ every $K$ gradient updates. In practice, we use $\pi_{\theta_\mathrm{old}}$ to generate a batch of denoising trajectories, and then use all pairs of denoising trajectories in the batch to compute the loss $\mathcal{L}(\theta)$. Details are provided in \cref{alg:training}. We will show in \cref{sec:online} that online optimization significantly improves generation quality.

\subsection{Proximal Updates for Stable Training}
\label{sec:clipping}

We find in our experiments that directly optimizing \cref{eqn:mse_loss} is prone to training instability, as illustrated in \cref{fig:ablation_clip} (Left). This is likely due to excessively large model updates during training. To resolve this issue, we propose proximal updates that remove the incentive for moving $\pi_\theta$ too far away from $\pi_{\theta_\mathrm{old}}$. Inspired by PPO~\citep{schulman2017ppo}, we achieve this by clipping the log probability ratio $\log \! \left( \pi_{\theta}(\bar{\rvx} | \rvc) / \pi_{\theta_\mathrm{old}}(\bar{\rvx} | \rvc) \right)$ to be within a small interval $[-\epsilon', \epsilon']$. This can be implemented by clipping the $\hat{r}_\theta(\bar{\rvx}, \rvc)$ as $\hat{r}_\theta^\mathrm{clip}(\bar{\rvx}, \rvc) \coloneqq$
\begin{align}
    \label{eqn:clipping}
    \mathrm{clip} \left( \hat{r}_\theta(\bar{\rvx}, \rvc), \hat{r}_{\theta_\mathrm{old}}(\bar{\rvx}, \rvc) - \epsilon', \hat{r}_{\theta_\mathrm{old}}(\bar{\rvx}, \rvc) + \epsilon' \right),
\end{align}
because $\log \! \left( \pi_{\theta}(\bar{\rvx} | \rvc) / \pi_{\theta_\mathrm{old}}(\bar{\rvx} | \rvc) \right) = \hat{r}_\theta(\bar{\rvx}, \rvc) - \hat{r}_{\theta_\mathrm{old}}(\bar{\rvx}, \rvc)$.
We then use $\hat{r}_\theta^\mathrm{clip}(\bar{\rvx}, \rvc)$ to compute the clipped MSE loss $l_\theta^\mathrm{clip}(\bar{\rvx}^a, \bar{\rvx}^b, \rvc) \coloneqq$
\begin{align}
    \left\lVert \Delta \hat{r}_\theta^\mathrm{clip}(\bar{\rvx}^a, \bar{\rvx}^b, \rvc) - \Delta r(\rvx_0^a, \rvx_0^b, \rvc) / \beta \right\rVert^2,
\end{align}
where $\Delta \hat{r}_\theta^\mathrm{clip}(\bar{\rvx}^a, \bar{\rvx}^b, \rvc) \coloneqq \hat{r}_\theta^\mathrm{clip}(\bar{\rvx}^a, \rvc) - \hat{r}_\theta^\mathrm{clip}(\bar{\rvx}^b, \rvc)$.
Similar to PPO~\citep{schulman2017ppo}, our final loss is the maximum of the clipped and unclipped MSE loss:
\begin{align}
    l_\theta(\bar{\rvx}^a, \bar{\rvx}^b, \rvc) \gets \max ( l_\theta(\bar{\rvx}^a, \bar{\rvx}^b, \rvc), l_\theta^\mathrm{clip}(\bar{\rvx}^a, \bar{\rvx}^b, \rvc) ).
\end{align}
This ensures that we minimize an upper bound of the original loss, making the optimization problem well-defined.

In practice, the clipping in \cref{eqn:clipping} is decomposed and applied at each denoising step $t$. First, $\hat{r}_\theta(\bar{\rvx}, \rvc)$ can be decomposed as $\hat{r}_\theta(\bar{\rvx}, \rvc) = \sum_{t=1}^T \hat{r}_{\theta, t}(\bar{\rvx}, \rvc)$, where
\begin{align}
    \hat{r}_{\theta, t}(\bar{\rvx}, \rvc) \coloneqq \log \! \left( \pi_{\theta}(\rvx_{t-1} | \rvx_t, \rvc) / \pi_\mathrm{ref}(\rvx_{t-1} | \rvx_t, \rvc) \right).
\end{align}
We apply clipping to each $\hat{r}_{\theta, t}(\bar{\rvx}, \rvc)$ as $\hat{r}_{\theta, t}^\mathrm{clip}(\bar{\rvx}, \rvc) \coloneqq$
\begin{align}
    \mathrm{clip} \left( \hat{r}_{\theta, t}(\bar{\rvx}, \rvc), \hat{r}_{\theta_\mathrm{old}, t}(\bar{\rvx}, \rvc) - \epsilon, \hat{r}_{\theta_\mathrm{old}, t}(\bar{\rvx}, \rvc) + \epsilon \right),
\end{align}
where $\epsilon$ is the stepwise clipping range. Finally, we replace \cref{eqn:clipping} with
\begin{align}
    \hat{r}_\theta^\mathrm{clip}(\bar{\rvx}, \rvc) &\coloneqq \sum_{t=1}^T \hat{r}_{\theta, t}^\mathrm{clip}(\bar{\rvx}, \rvc).
\end{align}
As shown in \cref{fig:ablation_clip} (Right), our proposed proximal updates can remarkably improve optimization stability.

\begin{figure*}[t]
    \centering
    \includegraphics[width=\textwidth]{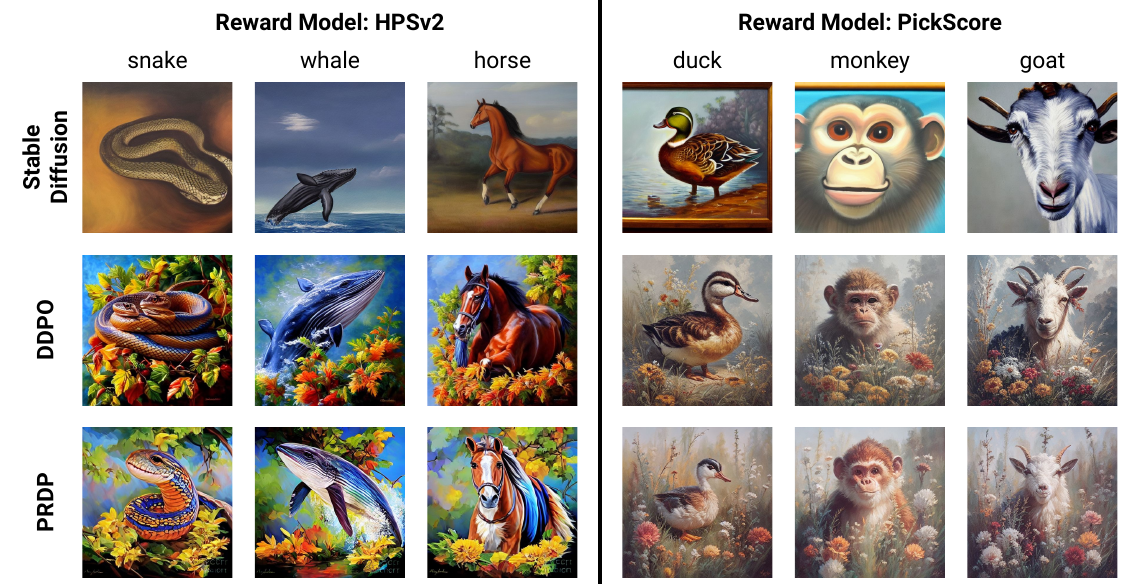}
    \caption{\textbf{Generation samples from small-scale training.} DDPO and PRDP are finetuned from Stable Diffusion v1.4 on $45$ prompts consisting of common animal names, with HPSv2 (Left) and PickScore (Right) as the reward model. Samples within each column use the same random seed. The prompt template is ``A painting of a $\langle$animal$\rangle$'', where the $\langle$animal$\rangle$ is listed on top of each column. All prompts are seen during training. Both DDPO and PRDP significantly improve the generation quality, with PRDP being slightly better.}
    \label{fig:common_animals}
\end{figure*}

\section{Experiments}
\label{sec:exp}

In our experiments, we first verify on a set of $45$ prompts that PRDP can match the reward maximization ability of DDPO~\citep{ddpo}, which is based on the well-established PPO~\citep{schulman2017ppo} algorithm. We then conduct a large-scale training on more than $100$K prompts from the training set of HPDv2~\citep{HPSv2}, showing that PRDP can successfully handle large-scale training whereas DDPO fails. We further perform a large-scale multi-reward finetuning on the training set prompts of Pick-a-Pic v1 dataset~\citep{PickScore}, highlighting the superior generation quality of PRDP on complex, unseen prompts. Finally, we showcase the advantages of our algorithm design, such as online optimization and KL regularization.

\subsection{Experimental Setup}
\label{sec:exp_setup}

To perform reward finetuning, we need a pretrained diffusion model, a pretrained reward model, and a training set of prompts. For all experiments, we use Stable Diffusion (SD) v1.4~\citep{stable-diffusion} as the pretrained diffusion model, and finetune the full UNet weights. For sampling, during both training and evaluation, we use the DDPM sampler~\citep{ddpm} with $50$ denoising steps and a classifier-free guidance~\citep{ho2021classifierfree} scale of $5.0$.

\textbf{Small-scale setup.}
We use a set of $45$ prompts, with the template ``A painting of a $\langle$animal$\rangle$'', where the $\langle$animal$\rangle$ is taken from the list of common animal names used in DDPO. We conduct reward finetuning separately for two recently proposed reward models, HPSv2~\citep{HPSv2} and PickScore~\citep{PickScore}. We train for $100$ epochs, where in each epoch, we sample $32$ prompts and $16$ images per prompt. The evaluation uses the same set of prompts as training. We report reward scores averaged over $256$ random samples per prompt.

\begin{table}[t]
  \caption{\textbf{Reward score comparison on small-scale training.}}
  \label{tab:common_animals}
  \centering
  \begin{adjustbox}{max width=\columnwidth}
  \begin{tabular}{cccc}
    \toprule
              & SD v1.4  & DDPO     & PRDP \\
    \midrule
    HPSv2     & $0.2855$ & $0.3398$ & $\mathbf{0.3471}$ \\
    \midrule
    PickScore & $0.2179$ & $0.2664$ & $\mathbf{0.2700}$ \\
    \bottomrule
  \end{tabular}
  \end{adjustbox}
\end{table}

\textbf{Large-scale setup.}
Following DRaFT~\citep{draft}, we use more than $100$K prompts from the training set of HPDv2, and finetune for HPSv2 and PickScore separately. We train for $1000$ epochs. In each epoch, we sample $64$ prompts and $8$ images per prompt. We evaluate the finetuned model on $500$ randomly sampled training prompts, as well as a variety of unseen prompts, including $500$ prompts from the Pick-a-Pic v1 test set, and $800$ prompts from each of the four benchmark categories of HPDv2, namely animation, concept art, painting, and photo. We report reward scores averaged over $64$ random samples per prompt.

\begin{figure*}[t]
    \centering
    \includegraphics[width=\textwidth]{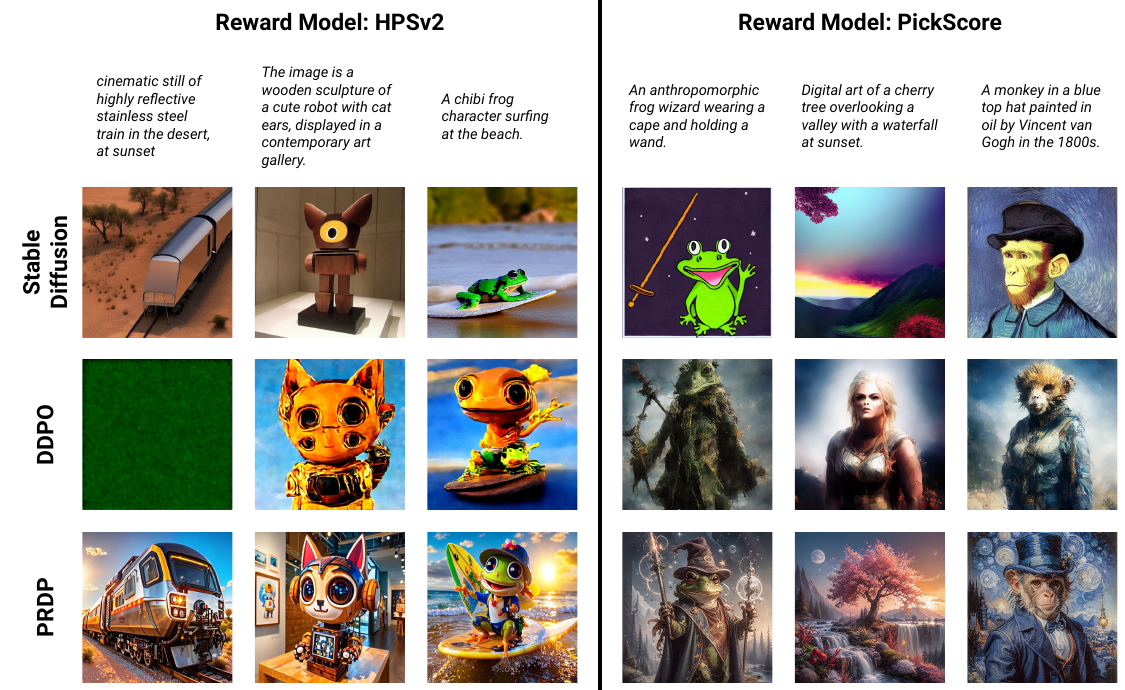}
    \caption{\textbf{Generation samples from large-scale training.} DDPO and PRDP are finetuned from Stable Diffusion v1.4 on over $100$K prompts from the training set of HPDv2, with HPSv2 (Left) and PickScore (Right) as the reward model. Samples within each column are generated from the prompt shown on top, using the same random seed. All prompts are unseen during training. PRDP significantly improves the generation quality over Stable Diffusion, whereas DDPO fails to generate reasonable results.}
    \label{fig:large_scale_samples}
\end{figure*}

\textbf{Large-scale multi-reward setup.}
We mostly follow the large-scale setup, except that we use the training set prompts of Pick-a-Pic v1 dataset, and a weighted combination of rewards: PickScore $= 10$, HPSv2 $= 2$, Aesthetic $= 0.05$, where Aesthetic is the LAION aesthetic score.

\textbf{Baselines.}
DDPO~\citep{ddpo} and DPOK~\citep{dpok} are the two most recent RL finetuning methods for black-box rewards. Since DDPO has demonstrated better performance than DPOK, we mainly compare to DDPO. To ensure a fair comparison, we train DDPO and PRDP for the same number of epochs, with the same number of reward queries per epoch. We also use the same random seeds to sample images for evaluation.

\begin{table*}[t]
  \caption{\textbf{Reward score comparison on large-scale training.}}
  \label{tab:large_scale}
  \centering
  \begin{adjustbox}{max width=\textwidth}
  \begin{tabular}{clcccccc}
    \toprule
    \multirow{2}{*}[-0.8em]{\makecell{Reward\\Model}} & \multirow{2}{*}[-0.8em]{Method} & Seen Prompts & \multicolumn{5}{c}{Unseen Prompts} \\ \cmidrule(lr){3-3} \cmidrule(lr){4-8} && \makecell{HPD v2\\Training Set} & \makecell{Pick-a-Pic v1\\Test Set} & \makecell{HPD v2\\Animation} & \makecell{HPD v2\\Concept Art} & \makecell{HPD v2\\Painting} & \makecell{HPD v2\\Photo} \\
    \midrule
    \multirow{3}{*}{HPSv2}
        & SD v1.4 & $0.2685$ & $0.2665$ & $0.2737$ & $0.2656$ & $0.2654$ & $0.2750$ \\
        & DDPO                  & $0.2464$ & $0.2501$ & $0.2673$ & $0.2558$ & $0.2570$ & $0.2093$ \\
        & PRDP                  & $\mathbf{0.3175}$ & $\mathbf{0.3050}$ & $\mathbf{0.3223}$ & $\mathbf{0.3175}$ & $\mathbf{0.3172}$ & $\mathbf{0.3159}$ \\
    \midrule
    \multirow{3}{*}{PickScore}
        & SD v1.4 & $0.2092$ & $0.2082$ & $0.2111$ & $0.2062$ & $0.2059$ & $0.2172$ \\
        & DDPO                  & $0.2032$ & $0.1992$ & $0.2077$ & $0.2125$ & $0.2124$ & $0.1780$ \\
        & PRDP                  & $\mathbf{0.2424}$ & $\mathbf{0.2344}$ & $\mathbf{0.2450}$ & $\mathbf{0.2441}$ & $\mathbf{0.2448}$ & $\mathbf{0.2387}$ \\
    \bottomrule
  \end{tabular}
  \end{adjustbox}
\end{table*}

\begin{figure*}[t]
    \centering
    \includegraphics[width=0.97\textwidth]{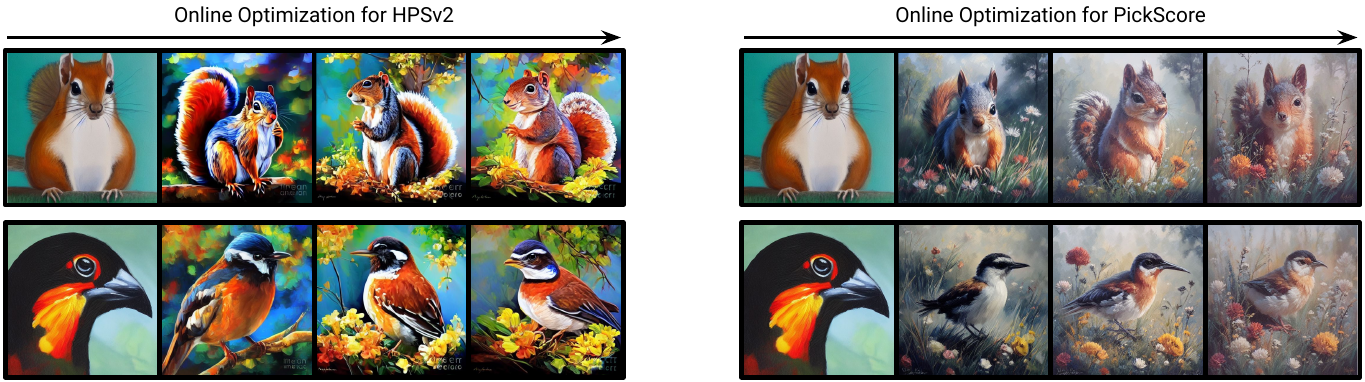}
    \caption{\textbf{Effect of online optimization.} We show generation samples during the PRDP training process, with HPSv2 (Left) and PickScore (Right) as the reward model. We follow the small-scale training setup. The prompts for the first and the second rows are ``A painting of a squirrel'' and ``A painting of a bird'', respectively. Samples within each row use the same random seed. It can be observed that online optimization continually improves the generation quality.}
    \label{fig:online_samples}
\end{figure*}

\subsection{Main Results}

\textbf{Small-scale finetuning.}
We show generation samples from small-scale finetuning in \cref{fig:common_animals} and reward scores in \cref{tab:common_animals}. Both DDPO and PRDP can significantly improve the generation quality over Stable Diffusion, with more vivid colors and details. Quantitatively, PRDP achieves slightly better reward scores than DDPO. This verifies that PRDP can match the reward maximization ability of well-established policy gradient methods.

\textbf{Large-scale finetuning.}
We present generation samples from large-scale finetuning in \cref{fig:large_scale_samples} and reward scores in \cref{tab:large_scale}. We observe that Stable Diffusion generates images with relevant content but low quality. Meanwhile, DDPO fails to give reasonable results. It generates irrelevant, low quality images or even meaningless noise, leading to lower reward scores than Stable Diffusion. This is due to the instability of DDPO in large-scale training, which we further investigate in \cref{sec:ddpo_instability}. In contrast, PRDP maintains stability in the large-scale setup, and significantly improves the generation quality on both seen and unseen prompts.

\textbf{Large-scale multi-reward finetuning.}
We provide generation samples in \cref{fig:teaser,fig:pick_a_pic_test,fig:anime,fig:concept_art,fig:painting,fig:photo}, and reward scores in \cref{tab:multi_reward}, showing the superior generation quality of PRDP on a diverse set of complex, unseen prompts.

\subsection{Effect of Online Optimization}
\label{sec:online}

In this section, we show that online optimization has a great advantage over offline optimization. To ensure a fair comparison, we use the same number of reward queries and gradient updates for both methods. Specifically, following the small-scale setup, for online training, we use $100$ epochs, where each epoch makes $512$ queries to the reward model. For offline training, we sample $51200$ images from the pretrained Stable Diffusion, obtain their rewards, and then perform the same total number of gradient updates as in online training. We show generation samples during the online optimization process in \cref{fig:online_samples}, and quantitative comparisons in \cref{fig:online_training}. We observe that online optimization continually improves the generation quality, achieving significantly better reward scores than offline optimization.

\begin{figure}[t]
    \centering
    \includegraphics[width=\columnwidth]{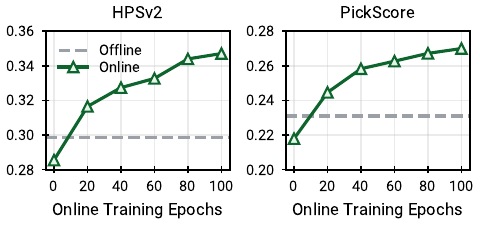}
    \caption{\textbf{Comparison of online and offline optimization.} We evaluate the reward scores of model checkpoints during online optimization and the final model obtained by offline optimization. We follow the small-scale training setup, and optimize the models for HPSv2 and PickScore separately. Online optimization matches the performance of offline optimization in ${\sim} 10$ epochs, and keeps improving the reward score afterwards.}
    \label{fig:online_training}
\end{figure}

\begin{figure}[t]
    \centering
    \includegraphics[width=\columnwidth]{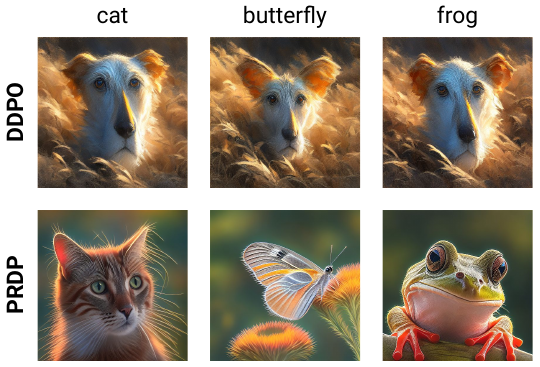}
    \caption{\textbf{Effect of KL regularization.} We show generation samples from DDPO and PRDP when optimizing the LAION aesthetic score. We use the small-scale training setup, except that we train for $250$ epochs. Samples within each column are generated from the prompt shown on top, using the same random seed. DDPO, without KL regularization, over-optimizes the reward, generating similar images for all prompts. In contrast, PRDP, formulated with KL regularization, successfully preserves text-image alignment.}
    \label{fig:kl}
\end{figure}

\subsection{Effect of KL Regularization}

A common limitation of reward finetuning is reward hacking, where the finetuned diffusion model exploits inaccuracies in the reward model, and produces undesired images with high reward scores. In this section, we show that the KL regularization in our PRDP formulation can help alleviate this issue. For this purpose, we use the LAION aesthetic predictor as the reward model. It only takes images as input, and can be exploited by disregarding text-image alignment. We follow the small-scale setup, except that we train for $250$ epochs and directly use the $45$ common animal names as prompts. As demonstrated in \cref{fig:kl}, DDPO, without KL regularization, is prone to reward hacking. It completely ignores the text prompts and generates similar images for all prompts. In contrast, PRDP with $\beta = 10$ can successfully preserve the text-image alignment while improving the aesthetic quality. More analysis can be found in \cref{sec:kl}.

\section{Related Work}
\label{sec:related_work}

\textbf{Diffusion models.}
As a new class of generative models, diffusion models \citep{sohl2015deep,ddpm,song2021scorebased} have achieved remarkable success in a wide variety of data modalities, including images \citep{adm,stable-diffusion,dalle2,imagen,glide,cascaded-diffusion,sr3,palette}, videos \citep{imagen-video,singer2022make}, audios~\citep{AudioLDM}, 3D shapes \citep{zhou20213d,zeng2022lion,DreamFusion,NerfDiff}, and robotic trajectories \citep{janner2022planning,ajay2023is,chen2024simple}. To facilitate control over the content and style of generation, recent works have investigated finetuning diffusion models on various conditioning signals \citep{controlnet,mou2023t2i,ruiz2023dreambooth,textual-inversion,null-text,kawar2023imagic,styledrop,jiang2023objectcentric}. However, it remains challenging to adapt diffusion models to downstream use cases that are misaligned with the training objective, such as generating novel compositions of objects unseen during training, and producing images that are aesthetically preferred by humans. Although classifier guidance~\citep{adm} can help mitigate this issue, the classifier requires noisy images as input, making it hard to use off-the-shelf classifiers such as object detectors and aesthetic predictors for guidance. In contrast, we finetune the diffusion model to maximize rewards that reflect downstream objectives. Our method can work with generic off-the-shelf reward models that take clean images as input.

\textbf{Language model learning from human feedback.}
The maximum likelihood training objective for language models tends to yield undesirable model behavior, due to the potentially biased, toxic, or harmful content in the training data. Reinforcement learning from human feedback (RLHF) has recently emerged as a successful remedy \citep{ziegler2019fine,stiennon2020learning,WebGPT,InstructGPT,wu2021recursively,bai2022training,bai2022constitutional,glaese2022improving,liu2023chain}. Typically, a reward model is first trained from human preference data (\eg, rankings of outputs from a pretrained language model). Then, the language model is finetuned by online RL algorithms (\eg, PPO~\citep{schulman2017ppo}) to maximize the score given by the reward model. More recently, DPO~\citep{rafailov2023dpo} proposes a supervised learning method that directly optimizes the language model from preference data, skipping the reward model training and avoiding the instability of RL algorithms. Our method is inspired by DPO and PPO, but designed specifically for diffusion models.

\textbf{Reward finetuning for diffusion models.}
Inspired by the success of RLHF in the language domain, researchers have developed several reward models in the vision domain \citep{clip,blip,coca,vila,image-reward,wu2023human,HPSv2,PickScore,lee2023aligning}. Moreover, recent works have explored using these reward models to improve the generation quality of diffusion models. A simple approach, called supervised finetuning~\citep{lee2023aligning,wu2023human}, is to finetune the diffusion model toward high-reward samples from an offline dataset. Its major drawback is that the generation quality is limited by the offline dataset. For further improvement, RAFT~\citep{raft} proposes an online variant that iteratively re-generates the dataset. A more direct method for online optimization is to backpropagate the reward function gradient through the denoising process \citep{wallace2023end,image-reward,draft,prabhudesai2023aligning}. However, this only works for differentiable rewards. For generic rewards, DDPO~\citep{ddpo} and DPOK~\citep{dpok} propose RL finetuning. While they have shown promising results on small prompt sets, they are unstable in large-scale training. Our work addresses the training instability issue, achieving stable reward finetuning on large-scale prompt datasets for generic rewards. Concurrent with our work, Diffusion-DPO~\citep{Diffusion-DPO} adapts DPO to efficiently align diffusion models from large-scale offline preference data, and \citep{zhang2024large} proposes to stabilize large-scale RL finetuning by combining the diffusion model pretraining loss.

\section{Conclusion}
\label{sec:conclusion}

This paper presents PRDP, the first black-box reward finetuning method for diffusion models that is stable on large-scale prompt datasets with over $100$K prompts. We achieve this by converting the RLHF objective to an equivalent supervised regression objective and developing its stable optimization algorithm. Our large-scale experiments highlight the superior generation quality of PRDP on complex, unseen prompts, which is beyond the capability of existing RL finetuning methods. We also demonstrate that the KL regularization in the PRDP formulation can help alleviate the common issue of reward hacking. We hope that our work can inspire future research on large-scale reward finetuning for diffusion models.

\subsection*{Acknowledgments}

We thank authors of DRaFT~\citep{draft} for sharing their training prompts and reward models. We appreciate helpful discussion with Ligong Han, Yanwu Xu, Yaxuan Zhu, Zhonghao Wang, Yunzhi Zhang, Yang Zhao, and Zhisheng Xiao.

{
    \small
    \bibliographystyle{ieeenat_fullname}
    \bibliography{main}
}

\clearpage
\maketitlesupplementary
\appendix

\section{Proofs}
\label{sec:proof}

\subsection{Lower Bound of RLHF Objective}
\label{sec:lower_bound}

In \Cref{lem:lower_bound}, we prove that the objective in \Cref{eqn:objective} is a lower bound of the RLHF objective in \Cref{eqn:rlhf_objective}.

\begin{lemma}
\label{lem:lower_bound}
Given two diffusion models $\pi_\theta, \pi_\mathrm{ref}$, a prompt distribution $p(\rvc)$, a reward function $r(\rvx_0, \rvc)$, and a constant $\beta > 0$, we have:
\begin{align}
    & \E_{\rvc \sim p(\rvc)} \! \left[ \E_{\rvx_0 \sim \pi_\theta(\rvx_0 | \rvc)} [r(\rvx_0, \rvc)] - \beta \mathrm{KL}[\pi_\theta(\rvx_0 | \rvc) || \pi_\mathrm{ref}(\rvx_0 | \rvc)] \right] \\
    \geq{}& \E_{\rvc \sim p(\rvc)} \! \left[ \E_{\rvx_0 \sim \pi_\theta(\rvx_0 | \rvc)} [r(\rvx_0, \rvc)] - \beta \mathrm{KL}[\pi_\theta(\bar{\rvx} | \rvc) || \pi_\mathrm{ref}(\bar{\rvx} | \rvc)] \right],
\end{align}
where $\bar{\rvx} \coloneqq \rvx_{0:T}$ is the full denoising trajectory, and $\pi_\theta, \pi_\mathrm{ref}$ are defined as:
\begin{align}
    \pi(\rvx_0 | \rvc) = \int \pi(\rvx_{0:T} | \rvc) \, \mathrm{d}\rvx_{1:T} = \int p(\rvx_T) \prod_{t=1}^T \pi(\rvx_{t-1} | \rvx_t, \rvc) \, \mathrm{d}\rvx_{1:T}.
\end{align}
\end{lemma}

\begin{proof}
It suffices to show that for any $\rvc$,
\begin{align}
    \mathrm{KL}[\pi_\theta(\bar{\rvx} | \rvc) || \pi_\mathrm{ref}(\bar{\rvx} | \rvc)] \geq \mathrm{KL}[\pi_\theta(\rvx_0 | \rvc) || \pi_\mathrm{ref}(\rvx_0 | \rvc)].
\end{align}
This can be proved similarly as the data processing inequality. We provide the proof below.
\begin{align}
    \mathrm{KL}[\pi_\theta(\bar{\rvx} | \rvc) || \pi_\mathrm{ref}(\bar{\rvx} | \rvc)] &= \E_{\pi_\theta(\rvx_{0:T} | \rvc)} \! \left[ \log \frac{\pi_\theta(\rvx_{0:T} | \rvc)}{\pi_\mathrm{ref}(\rvx_{0:T} | \rvc)} \right] \\
    &= \E_{\pi_\theta(\rvx_{0:T} | \rvc)} \! \left[ \log \frac{\pi_\theta(\rvx_0 | \rvc)}{\pi_\mathrm{ref}(\rvx_0 | \rvc)} + \log \frac{\pi_\theta(\rvx_{1:T} | \rvx_0, \rvc)}{\pi_\mathrm{ref}(\rvx_{1:T} | \rvx_0, \rvc)} \right] \\
    &= \E_{\pi_\theta(\rvx_0 | \rvc)} \! \left[ \log \frac{\pi_\theta(\rvx_0 | \rvc)}{\pi_\mathrm{ref}(\rvx_0 | \rvc)} \right] + \E_{\pi_\theta(\rvx_0 | \rvc)} \! \left[ \E_{\pi_\theta(\rvx_{1:T} | \rvx_0, \rvc)} \! \left[ \log \frac{\pi_\theta(\rvx_{1:T} | \rvx_0, \rvc)}{\pi_\mathrm{ref}(\rvx_{1:T} | \rvx_0, \rvc)} \right] \right] \\
    &= \mathrm{KL}[\pi_\theta(\rvx_0 | \rvc) || \pi_\mathrm{ref}(\rvx_0 | \rvc)] + \E_{\pi_\theta(\rvx_0 | \rvc)} \! \left[ \mathrm{KL}[\pi_\theta(\rvx_{1:T} | \rvx_0, \rvc) || \pi_\mathrm{ref}(\rvx_{1:T} | \rvx_0, \rvc)] \right] \\
    &\geq \mathrm{KL}[\pi_\theta(\rvx_0 | \rvc) || \pi_\mathrm{ref}(\rvx_0 | \rvc)].
\end{align}
\end{proof}

\newpage
\subsection{Maximizer of the Lower Bound of RLHF Objective}
\label{sec:maximizer}

In \Cref{lem:maximizer}, we prove that \Cref{eqn:optimal_solution} maximizes the objective in \Cref{eqn:objective}, a lower bound of the RLHF objective.

\begin{lemma}
\label{lem:maximizer}
Define
\begin{align}
    \pi_{\theta^\star}(\bar{\rvx} | \rvc) = \frac{1}{Z(\rvc)} \pi_\mathrm{ref}(\bar{\rvx} | \rvc) \exp \! \left( \frac{1}{\beta} r(\rvx_0, \rvc) \right),
\end{align}
where
\begin{align}
    Z(\rvc) = \int \pi_\mathrm{ref}(\bar{\rvx} | \rvc) \exp \! \left( \frac{1}{\beta} r(\rvx_0, \rvc) \right) \mathrm{d}\bar{\rvx}
\end{align}
is the partition function. Then $\pi_{\theta^\star}$ is the optimal solution to the following maximization problem:
\begin{align}
\label{eqn:maximization_problem}
    \max_{\pi_\theta} \ \E_{\rvc \sim p(\rvc)} \! \left[ \E_{\rvx_0 \sim \pi_\theta(\rvx_0 | \rvc)} [r(\rvx_0, \rvc)] - \beta \mathrm{KL}[\pi_\theta(\bar{\rvx} | \rvc) || \pi_\mathrm{ref}(\bar{\rvx} | \rvc)] \right].
\end{align}
\end{lemma}

\begin{proof}
We provide the proof below, which is inspired by DPO~\citep{rafailov2023dpo}.
\begin{align}
    & \max_{\pi_\theta} \ \E_{\rvc \sim p(\rvc)} \! \left[ \E_{\rvx_0 \sim \pi_\theta(\rvx_0 | \rvc)} [r(\rvx_0, \rvc)] - \beta \mathrm{KL}[\pi_\theta(\bar{\rvx} | \rvc) || \pi_\mathrm{ref}(\bar{\rvx} | \rvc)] \right] \\
    ={}& \max_{\pi_\theta} \ \E_{\rvc \sim p(\rvc)} \! \left[ \E_{\bar{\rvx} \sim \pi_\theta(\bar{\rvx} | \rvc)} [r(\rvx_0, \rvc)] - \beta \mathrm{KL}[\pi_\theta(\bar{\rvx} | \rvc) || \pi_\mathrm{ref}(\bar{\rvx} | \rvc)] \right] \\
    ={}& \max_{\pi_\theta} \ \E_{\rvc \sim p(\rvc)} \E_{\bar{\rvx} \sim \pi_\theta(\bar{\rvx} | \rvc)} \! \left[ r(\rvx_0, \rvc) - \beta \log \frac{\pi_\theta(\bar{\rvx} | \rvc)}{\pi_\mathrm{ref}(\bar{\rvx} | \rvc)} \right] \\
    ={}& \min_{\pi_\theta} \ \E_{\rvc \sim p(\rvc)} \E_{\bar{\rvx} \sim \pi_\theta(\bar{\rvx} | \rvc)} \! \left[ \log \frac{\pi_\theta(\bar{\rvx} | \rvc)}{\pi_\mathrm{ref}(\bar{\rvx} | \rvc)} - \frac{1}{\beta} r(\rvx_0, \rvc) \right] \\
    ={}& \min_{\pi_\theta} \ \E_{\rvc \sim p(\rvc)} \E_{\bar{\rvx} \sim \pi_\theta(\bar{\rvx} | \rvc)} \! \left[ \log \frac{\pi_\theta(\bar{\rvx} | \rvc)}{\pi_\mathrm{ref}(\bar{\rvx} | \rvc) \exp \! \left( \frac{1}{\beta} r(\rvx_0, \rvc) \right)} \right] \\
    ={}& \min_{\pi_\theta} \ \E_{\rvc \sim p(\rvc)} \E_{\bar{\rvx} \sim \pi_\theta(\bar{\rvx} | \rvc)} \! \left[ \log \frac{\pi_\theta(\bar{\rvx} | \rvc)}{\pi_{\theta^\star}(\bar{\rvx} | \rvc) Z(\rvc)} \right] \\
    ={}& \min_{\pi_\theta} \ \E_{\rvc \sim p(\rvc)} \! \left[ \E_{\bar{\rvx} \sim \pi_\theta(\bar{\rvx} | \rvc)} \! \left[ \log \frac{\pi_\theta(\bar{\rvx} | \rvc)}{\pi_{\theta^\star}(\bar{\rvx} | \rvc)} \right] - \log Z(\rvc) \right] \\
    ={}& \min_{\pi_\theta} \ \E_{\rvc \sim p(\rvc)} \! \left[ \mathrm{KL}[\pi_\theta(\bar{\rvx} | \rvc) || \pi_{\theta^\star}(\bar{\rvx} | \rvc)] - \log Z(\rvc) \right] \\
    ={}& \min_{\pi_\theta} \ \E_{\rvc \sim p(\rvc)} \! \left[ \mathrm{KL}[\pi_\theta(\bar{\rvx} | \rvc) || \pi_{\theta^\star}(\bar{\rvx} | \rvc)] \right].
\end{align}
Since $\mathrm{KL}[\pi_\theta(\bar{\rvx} | \rvc) || \pi_{\theta^\star}(\bar{\rvx} | \rvc)] \geq 0$, and $\mathrm{KL}[\pi_\theta(\bar{\rvx} | \rvc) || \pi_{\theta^\star}(\bar{\rvx} | \rvc)] = 0$ if and only if $\pi_\theta(\bar{\rvx} | \rvc) = \pi_{\theta^\star}(\bar{\rvx} | \rvc)$, we conclude that the optimal solution to \Cref{eqn:maximization_problem} is $\pi_\theta(\bar{\rvx} | \rvc) = \pi_{\theta^\star}(\bar{\rvx} | \rvc)$ for all $\rvc$.
\end{proof}

\newpage
\subsection{Necessary and Sufficient Conditions for the Optimal Solution}
\label{sec:equivalent_condition}

In \Cref{lem:equivalent_condition}, we provide theoretical justification for our proposed RDP objective in \Cref{eqn:mse_loss}.

\begin{lemma}
\label{lem:equivalent_condition}
\begin{align}
    & \pi_\theta(\bar{\rvx} | \rvc) = \pi_{\theta^\star}(\bar{\rvx} | \rvc), \quad \forall \bar{\rvx}, \rvc \\ \label{eqn:equivalent_condition}
    \iff{}& \log \frac{\pi_\theta(\bar{\rvx}^a | \rvc)}{\pi_\mathrm{ref}(\bar{\rvx}^a | \rvc)} - \log \frac{\pi_\theta(\bar{\rvx}^b | \rvc)}{\pi_\mathrm{ref}(\bar{\rvx}^b | \rvc)} = \frac{r(\rvx_0^a, \rvc) - r(\rvx_0^b, \rvc)}{\beta}, \quad \forall \bar{\rvx}^a, \bar{\rvx}^b, \rvc.
\end{align}
\end{lemma}

\begin{proof}
We have shown ``$\implies$'' in the main text. We provide the proof for ``$\impliedby$'' below.

\Cref{eqn:equivalent_condition} implies that
\begin{align}
\label{eqn:const_diff}
    \log \frac{\pi_\theta(\bar{\rvx} | \rvc)}{\pi_\mathrm{ref}(\bar{\rvx} | \rvc)} - \frac{1}{\beta} r(\rvx_0, \rvc)
\end{align}
is a constant \wrt $\bar{\rvx}$. Therefore, we can write \Cref{eqn:const_diff} as a function of $\rvc$ alone:
\begin{align}
    \log \frac{\pi_\theta(\bar{\rvx} | \rvc)}{\pi_\mathrm{ref}(\bar{\rvx} | \rvc)} - \frac{1}{\beta} r(\rvx_0, \rvc) = f(\rvc).
\end{align}
Hence,
\begin{align}
    \pi_\theta(\bar{\rvx} | \rvc) = \pi_\mathrm{ref}(\bar{\rvx} | \rvc) \exp \! \left( \frac{1}{\beta} r(\rvx_0, \rvc) \right) \exp \! \left( f(\rvc) \right).
\end{align}
It suffices to show that
\begin{align}
    \exp \! \left( f(\rvc) \right) = \frac{1}{Z(\rvc)}, \quad \forall \rvc.
\end{align}
This follows from the fact that the probability density function $\pi_\theta(\bar{\rvx} | \rvc)$ must satisfy:
\begin{align}
    1 &= \int \pi_\theta(\bar{\rvx} | \rvc) \, \mathrm{d}\bar{\rvx} \\
    &= \int \pi_\mathrm{ref}(\bar{\rvx} | \rvc) \exp \! \left( \frac{1}{\beta} r(\rvx_0, \rvc) \right) \exp \! \left( f(\rvc) \right) \mathrm{d}\bar{\rvx} \\
    &= \exp \! \left( f(\rvc) \right) \int \pi_\mathrm{ref}(\bar{\rvx} | \rvc) \exp \! \left( \frac{1}{\beta} r(\rvx_0, \rvc) \right) \mathrm{d}\bar{\rvx} \\
    &= \exp \! \left( f(\rvc) \right) Z(\rvc).
\end{align}
\end{proof}

\newpage
\section{Instability of DDPO in Large-Scale Reward Finetuning}
\label{sec:ddpo_instability}

\begin{figure}[h]
    \centering
    \includegraphics[width=\textwidth]{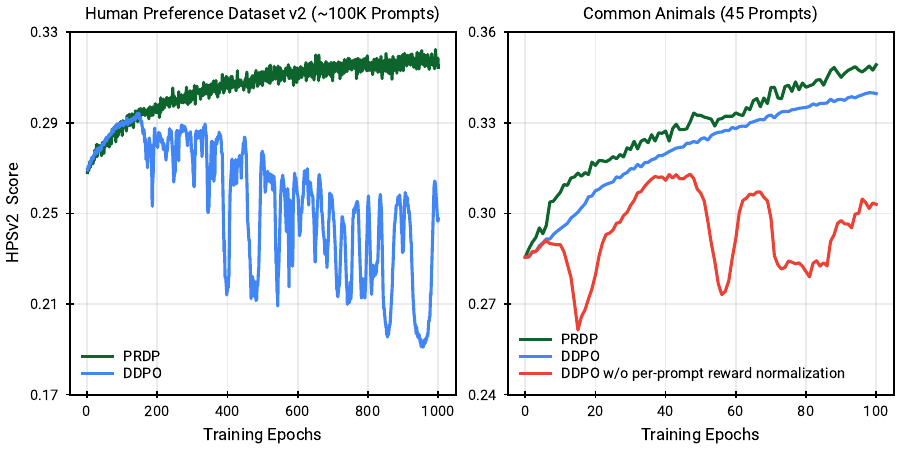}
    \caption{\textbf{Analysis of the instability of DDPO in large-scale training.} We plot the training curves of PRDP and DDPO on the large-scale Human Preference Dataset v2 (Left) and the small-scale Common Animals (Right). PRDP outperforms DDPO in the small-scale setting, and maintains stability in the large-scale setting where DDPO fails. Our ablation study suggests that the per-prompt reward normalization in DDPO is key to its stability, and the inability to perform such normalization in the large-scale setting likely causes its failure.}
    \label{fig:ddpo_ablation}
\end{figure}

\Cref{fig:ddpo_ablation} shows the training curve of PRDP and DDPO~\citep{ddpo}, where the reward model is HPSv2~\citep{HPSv2}. From \Cref{fig:ddpo_ablation} (Left), we observe that when trained on the large-scale Human Preference Dataset v2 (HPD v2)~\citep{HPSv2}, DDPO fails to stably optimize the reward. We conjecture that this is because the per-prompt reward normalization is rarely enabled in the large-scale setting, since each prompt can only be seen a few times. Specifically, in each epoch, DDPO randomly samples $512$ prompts, so on average, each prompt can be seen $512 \times 1000 / 100\text{K} \approx 5$ times. This is insufficient to obtain a good estimate of the per-prompt expected reward. In this case, DDPO will compute a prompt-agnostic expected reward, by averaging the rewards across all $512$ prompts. To verify that such prompt-agnostic reward normalization causes training instability, we conduct an ablation study of DDPO in our small-scale setting with $45$ training prompts. As shown in \Cref{fig:ddpo_ablation} (Right), DDPO without per-prompt reward normalization is unstable even in the small-scale setting, suggesting that the inability to perform per-prompt reward normalization can be a limiting factor in scaling DDPO to large prompt datasets. In contrast to DDPO, PRDP can steadily improve the reward score and maintain stability in both small-scale and large-scale settings.

\newpage
\section{Effect of KL Regularization}
\label{sec:kl}

\begin{figure}[h]
    \centering
    \includegraphics[width=\textwidth]{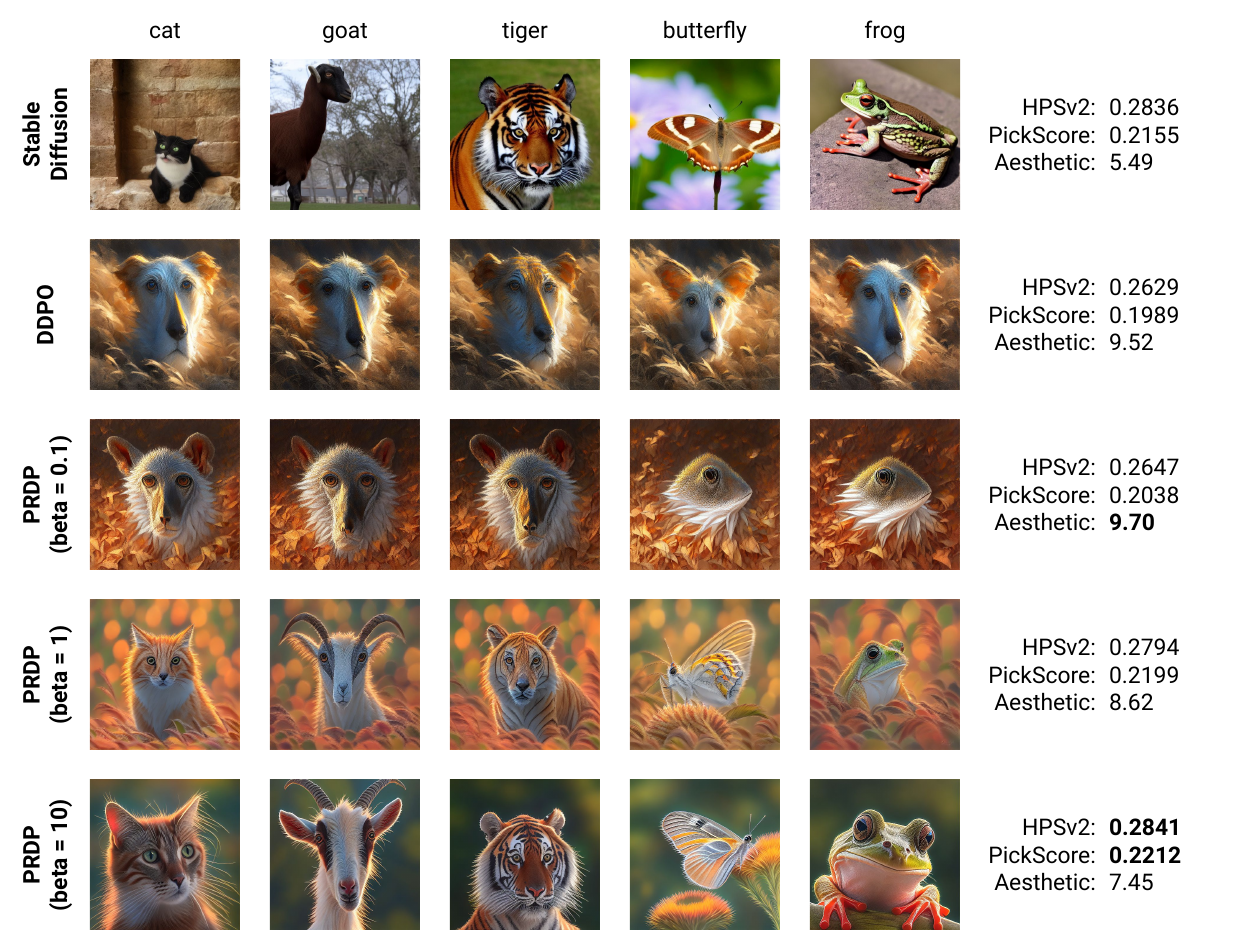}
    \caption{\textbf{Effect of KL regularization on optimizing aesthetic score.} DDPO and PRDP are finetuned from Stable Diffusion v1.4 on $45$ prompts of common animal names. Evaluation is performed on the same set of prompts. In addition to aesthetic score, we report HPSv2 and PickScore which reflect text-image alignment but are not used during training. Samples within each column are generated from the prompt shown on top, using the same random seed. PRDP with a large KL weight $\beta$ can alleviate the reward over-optimization problem encountered by DDPO, significantly improving the aesthetic quality over Stable Diffusion while maintaining text-image alignment.}
    \label{fig:beta_sweep}
\end{figure}

In contrast to DDPO~\citep{ddpo} which only cares about maximizing the reward, PRDP is formulated with a KL regularization, allowing us to alleviate the problem of reward over-optimization by increasing the KL weight $\beta$. We demonstrate the effect of KL regularization in \Cref{fig:beta_sweep}. Here, the reward used for training is the aesthetic score given by the LAION aesthetic predictor. It only takes images as input, and therefore ignores the text-image alignment. We finetune DDPO and PRDP from Stable Diffusion v1.4~\citep{stable-diffusion} for $250$ epochs on $45$ training prompts of common animal names as used in DDPO, with $512$ reward queries in each epoch. For evaluation, we additionally use HPSv2~\citep{HPSv2} and PickScore~\citep{PickScore} that reflect text-image alignment. The reported reward scores are averaged over $64$ random samples per training prompt, using the same random seed for Stable Diffusion v1.4, DDPO, and PRDP.

We observe that DDPO, without KL regularization, is prone to reward over-optimization. It ignores the text prompt and generates similar images for all prompts. PRDP with a small KL weight (\eg, $\beta = 0.1$) has the same problem, but achieves higher reward scores than DDPO, showing a better reward maximization capability. As the KL weight increases, PRDP is able to better preserve the text-image alignment, indicated by the increase in HPSv2 and PickScore. With $\beta = 10$, PRDP significantly improves the aesthetic score over Stable Diffusion v1.4 without sacrificing text-image alignment.

\newpage
\section{Large-Scale Multi-Reward Finetuning}

\begin{table}[h]
  \captionsetup{width=0.63\textwidth}
  \caption{\textbf{Reward score comparison on unseen prompts.} We use a weighted combination of rewards: PickScore $= 10$, HPSv2 $= 2$, Aesthetic $= 0.05$. PRDP is finetuned from Stable Diffusion v1.4 on the training set prompts of Pick-a-Pic v1 dataset.}
  \label{tab:multi_reward}
  \centering
  \begin{adjustbox}{max width=\textwidth}
  \begin{tabular}{lccccc}
    \toprule
    & \makecell{Pick-a-Pic v1\\Test Set} & \makecell{HPD v2\\Animation} & \makecell{HPD v2\\Concept Art} & \makecell{HPD v2\\Painting} & \makecell{HPD v2\\Photo} \\
    \midrule
    SD v1.4 & $2.888$ & $2.927$ & $2.877$ & $2.883$ & $2.984$ \\
    PRDP                  & $\mathbf{3.208}$ & $\mathbf{3.296}$ & $\mathbf{3.264}$ & $\mathbf{3.274}$ & $\mathbf{3.214}$ \\
    \bottomrule
  \end{tabular}
  \end{adjustbox}
\end{table}

In this section, we provide additional results for our large-scale multi-reward finetuning experiment. Following DRaFT~\citep{draft}, we use a weighted combination of rewards: PickScore $= 10$, HPSv2 $= 2$, Aesthetic $= 0.05$. We finetune Stable Diffusion v1.4~\citep{stable-diffusion} on the training set prompts of Pick-a-Pic v1 dataset~\citep{PickScore}. We evaluate our finetuned model on a variety of unseen prompts, including $500$ prompts from the Pick-a-Pic v1 test set, and $800$ prompts from each of the four benchmark categories of the Human Preference Dataset v2 (HPD v2)~\citep{HPSv2}, namely animation, concept art, painting, and photo. \Cref{tab:multi_reward} reports the reward scores before and after finetuning. The reward scores are averaged over $64$ random samples per prompt, using the same random seed for Stable Diffusion v1.4 and PRDP. We further show generation samples for each test prompt set in \Cref{fig:pick_a_pic_test,fig:anime,fig:concept_art,fig:painting,fig:photo}. As can be seen, PRDP significantly improves generation quality across all five prompt sets.

\section{Hyperparameters}

\begin{table}[h]
  \caption{\textbf{PRDP training hyperparameters.}}
  \centering
  \begin{adjustbox}{max width=\textwidth}
  \begin{tabular}{lcccc}
    \toprule
    \textbf{Name} & \textbf{Symbol} & \textbf{\makecell{Small-Scale\\Finetuning}} & \textbf{\makecell{Large-Scale\\Finetuning}} & \textbf{\makecell{Large-Scale Multi-Reward\\Finetuning}} \\
    \midrule
    Training epochs & $E$ & $100$ & $1000$ & $1000$ \\
    Gradient updates per epoch & $K$ & $10$ & $1$ & $1$ \\
    Prompts per epoch & $N$ & $32$ & $64$ & $64$ \\
    Images per prompt & $B$ & $16$ & $8$ & $8$ \\
    KL weight & $\beta$ & $3 {\times} 10^{-5}$ & $3 {\times} 10^{-6}$ & $3 {\times} 10^{-5}$ \\
    DDPM steps & $T$ & $50$ & $50$ & $50$ \\
    Stepwise clipping range & $\epsilon$ & $1 {\times} 10^{-6}$ & $1 {\times} 10^{-4}$ & $1 {\times} 10^{-4}$ \\
    Classifier-free guidance scale & --- & $5.0$ & $5.0$ & $5.0$ \\
    Optimizer & --- & AdamW & AdamW & AdamW \\
    Gradient clipping & --- & $1.0$ & $1.0$ & $1.0$ \\
    Learning rate & --- & $1 {\times} 10^{-5}$ & $7 {\times} 10^{-6}$ & $1 {\times} 10^{-5}$ \\
    Weight decay & --- & $1 {\times} 10^{-4}$ & $1 {\times} 10^{-4}$ & $1 {\times} 10^{-4}$ \\
    \bottomrule
  \end{tabular}
  \end{adjustbox}
\end{table}

\newpage
\section{Effect of Clipping}

\begin{table}[h]
  \caption{\textbf{Effect of clipping on training stability.}}
  \label{tab:clipping}
  \centering
  \begin{adjustbox}{max width=\textwidth}
  \begin{tabular}{ccc}
    \toprule
    & w/o Clipping & w/ Clipping \\
    \midrule
    DDPO & \makecell[l]{Small scale: \Unstable\\Large scale: \Unstable} & \makecell[l]{Small scale: \Stable\\Large scale: \Unstable} \\
    \midrule
    PRDP & \makecell[l]{Small scale: \Unstable\\Large scale: \Unstable} & \makecell[l]{Small scale: \Stable\\Large scale: \Stable} \\
    \bottomrule
  \end{tabular}
  \end{adjustbox}
\end{table}

\Cref{tab:clipping} summarizes the effect of clipping on the training stability of both DDPO~\citep{ddpo} and PRDP. For DDPO, we use PPO-based clipping~\citep{schulman2017ppo}, while for PRDP, we use the proximal updates described in \Cref{sec:clipping}. We observe that clipping is key to stability of small-scale training, whereas using the PRDP objective and clipping are both indispensable for achieving stability in large-scale training.

\section{Jax Implementation of PRDP Loss}

\begin{lstlisting}
import jax
import jax.numpy as jnp


def prdp_loss(
    log_probs: jax.Array,        # (B, T)
    log_probs_old: jax.Array,    # (B, T)
    log_probs_ref: jax.Array,    # (B, T)
    rewards: jax.Array,          # (B,)
    clip_range: float,
    kl_weight: float,
) -> jax.Array:
  """Computes PRDP loss for a batch of denoising trajectories with the same text prompt.

  Args:
    log_probs: Log probs of the denoising trajectories under pi_theta.
    log_probs_old: Log probs of the denoising trajectories under pi_theta_old.
    log_probs_ref: Log probs of the denoising trajectories under pi_ref.
    rewards: Rewards of the generated clean images.
    clip_range: Stepwise clipping range (epsilon).
    kl_weight: KL weight (beta).

  Returns:
    loss: The PRDP loss.
  """
  log_ratios = log_probs - log_probs_ref
  log_ratios_old = log_probs_old - log_probs_ref
  clipped_log_ratios = jnp.clip(
      log_ratios, log_ratios_old - clip_range, log_ratios_old + clip_range
  )

  log_ratios = jnp.mean(log_ratios, axis=-1)
  clipped_log_ratios = jnp.mean(clipped_log_ratios, axis=-1)

  log_ratio_diffs = log_ratios[:, None] - log_ratios
  clipped_log_ratio_diffs = clipped_log_ratios[:, None] - clipped_log_ratios
  reward_diffs = rewards[:, None] - rewards

  mse_loss = (log_ratio_diffs - reward_diffs / kl_weight) ** 2
  clipped_mse_loss = (clipped_log_ratio_diffs - reward_diffs / kl_weight) ** 2
  loss = jnp.maximum(mse_loss, clipped_mse_loss)
  loss = jnp.mean(loss, where=reward_diffs > 0)

  return loss
\end{lstlisting}

\newpage
\begin{figure}[p]
    \centering
    \includegraphics[width=\textwidth]{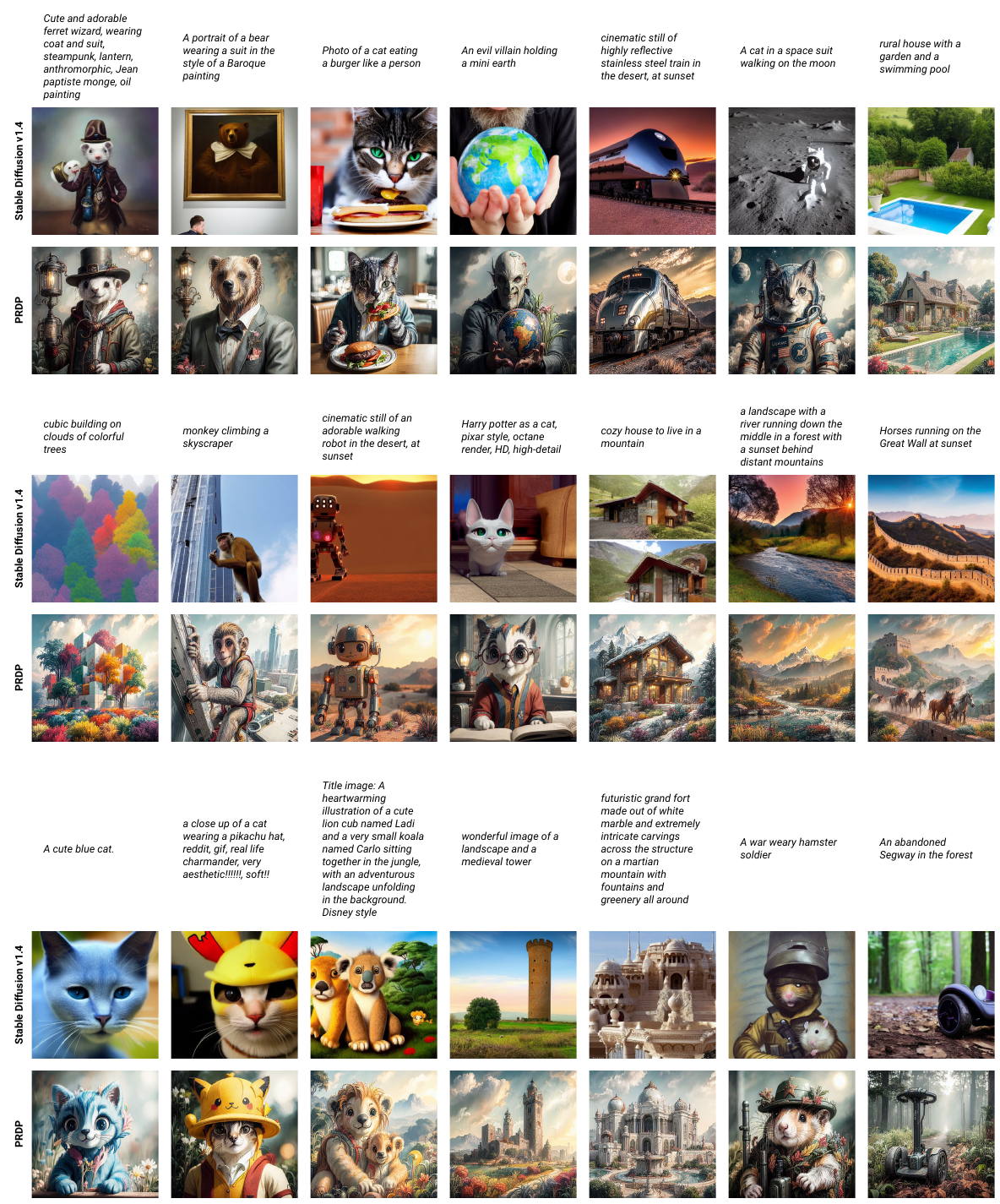}
    \caption{\textbf{Generation samples on unseen prompts from the Pick-a-Pic v1 test set.} PRDP is finetuned from Stable Diffusion v1.4 on the training set prompts of Pick-a-Pic v1 dataset, using a weighted combination of rewards: PickScore $= 10$, HPSv2 $= 2$, Aesthetic $= 0.05$. For each prompt, the generation sample from Stable Diffusion v1.4 and PRDP use the same random seed.}
    \label{fig:pick_a_pic_test}
\end{figure}

\newpage
\begin{figure}[p]
    \centering
    \includegraphics[width=\textwidth]{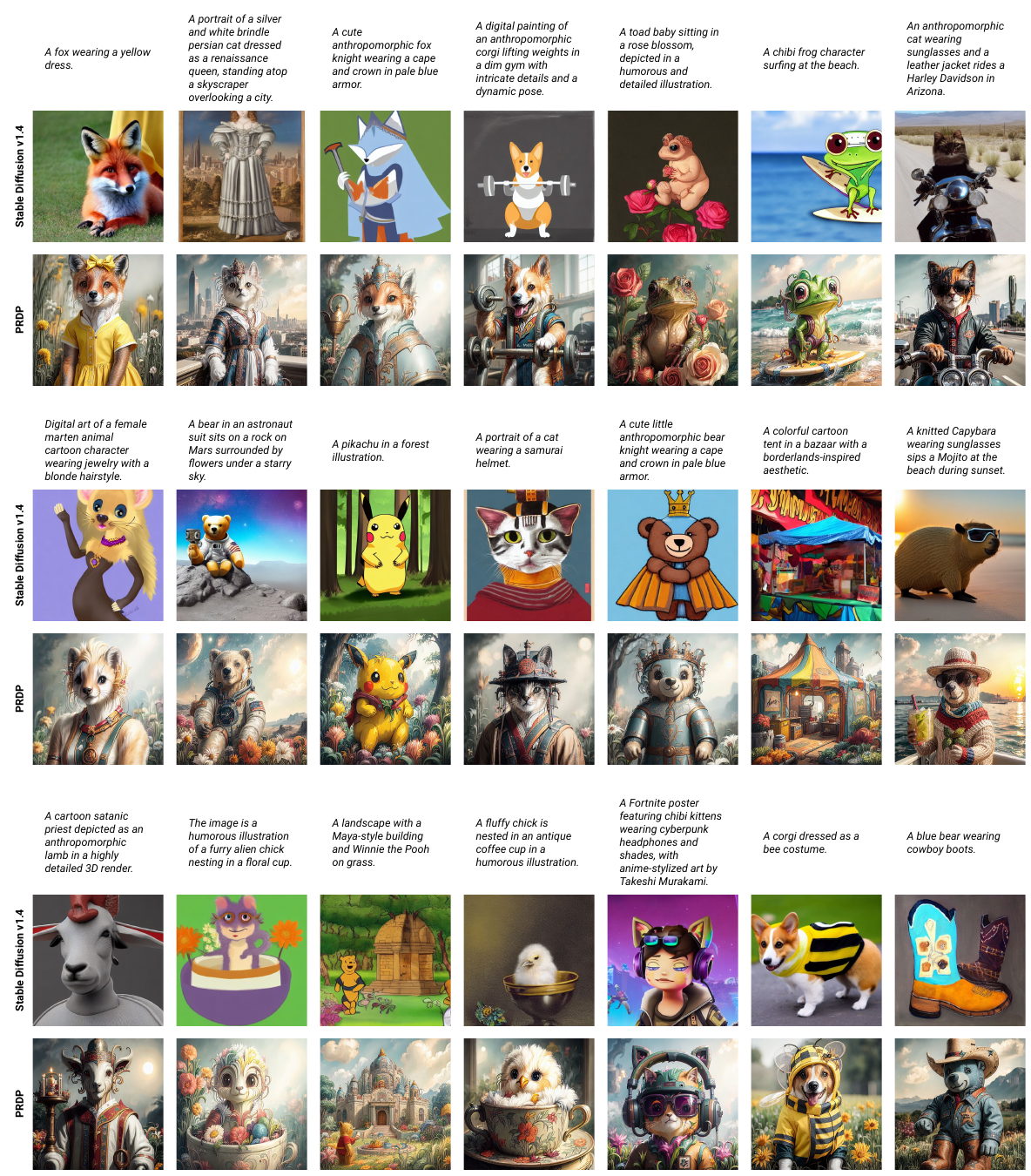}
    \caption{\textbf{Generation samples on unseen prompts from the HPD v2 animation benchmark.} PRDP is finetuned from Stable Diffusion v1.4 on the training set prompts of Pick-a-Pic v1 dataset, using a weighted combination of rewards: PickScore $= 10$, HPSv2 $= 2$, Aesthetic $= 0.05$. For each prompt, the generation sample from Stable Diffusion v1.4 and PRDP use the same random seed.}
    \label{fig:anime}
\end{figure}

\newpage
\begin{figure}[p]
    \centering
    \includegraphics[width=\textwidth]{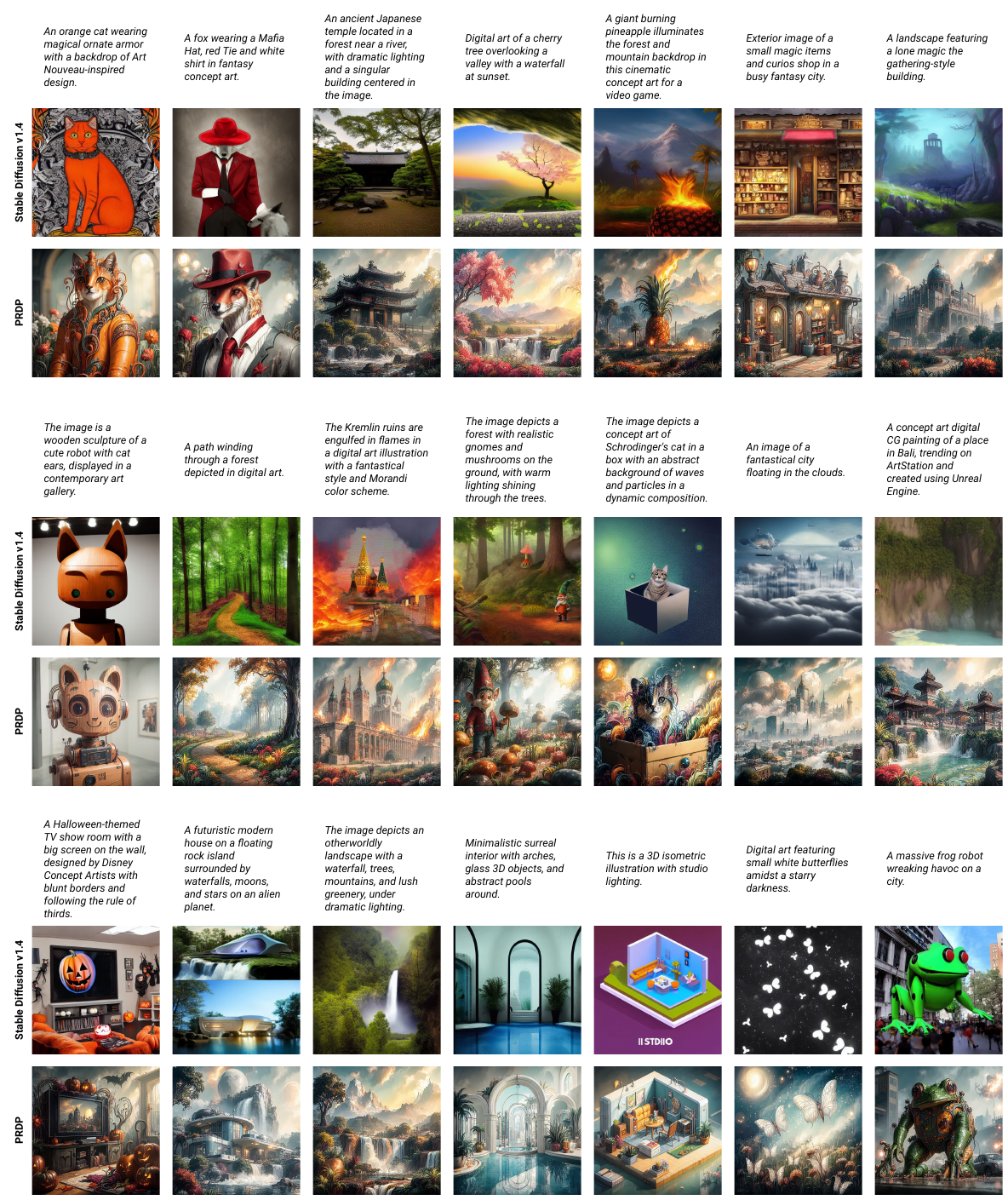}
    \caption{\textbf{Generation samples on unseen prompts from the HPD v2 concept art benchmark.} PRDP is finetuned from Stable Diffusion v1.4 on the training set prompts of Pick-a-Pic v1 dataset, using a weighted combination of rewards: PickScore $= 10$, HPSv2 $= 2$, Aesthetic $= 0.05$. For each prompt, the generation sample from Stable Diffusion v1.4 and PRDP use the same random seed.}
    \label{fig:concept_art}
\end{figure}

\newpage
\begin{figure}[p]
    \centering
    \includegraphics[width=\textwidth]{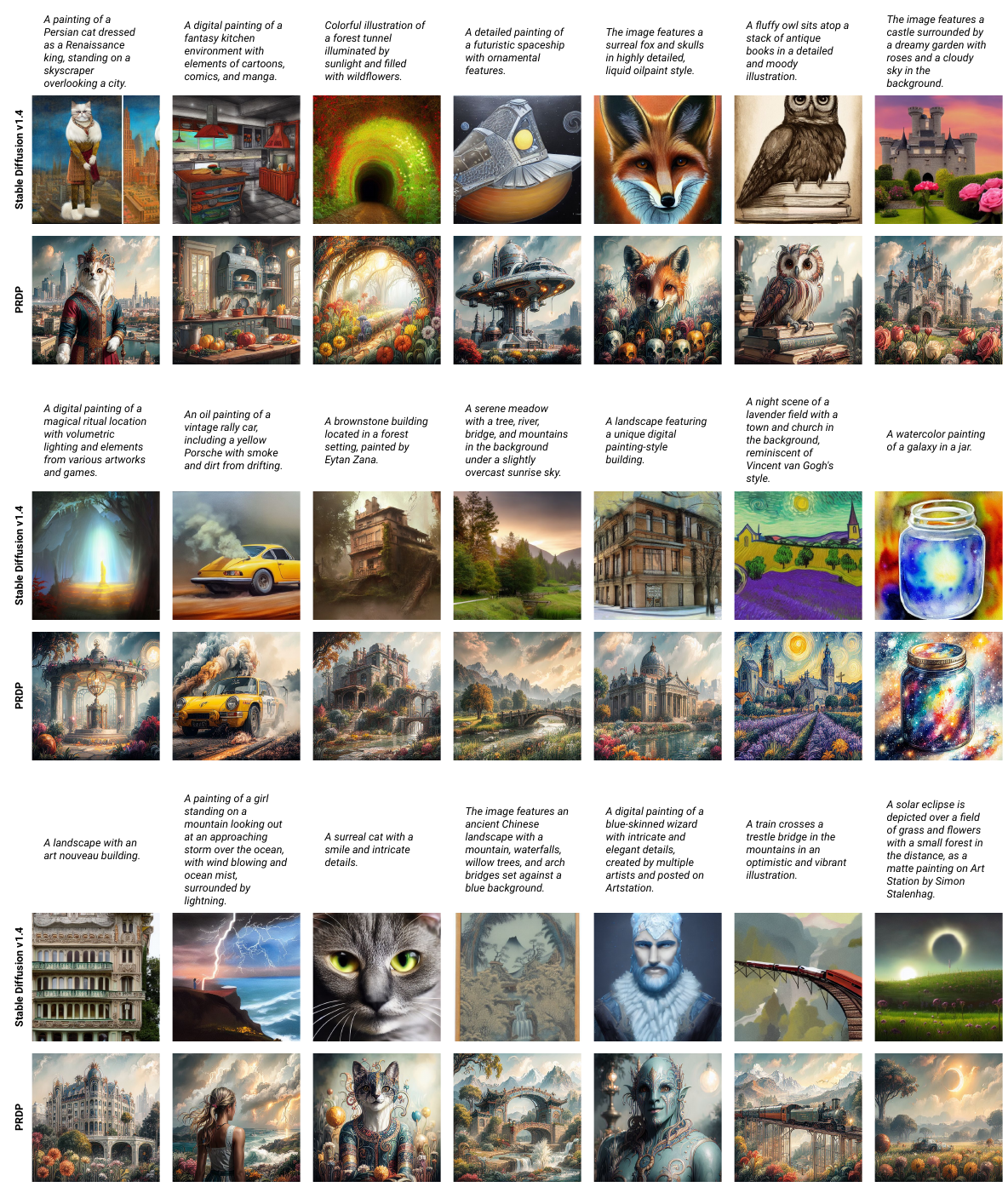}
    \caption{\textbf{Generation samples on unseen prompts from the HPD v2 painting benchmark.} PRDP is finetuned from Stable Diffusion v1.4 on the training set prompts of Pick-a-Pic v1 dataset, using a weighted combination of rewards: PickScore $= 10$, HPSv2 $= 2$, Aesthetic $= 0.05$. For each prompt, the generation sample from Stable Diffusion v1.4 and PRDP use the same random seed.}
    \label{fig:painting}
\end{figure}

\newpage
\begin{figure}[p]
    \centering
    \includegraphics[width=\textwidth]{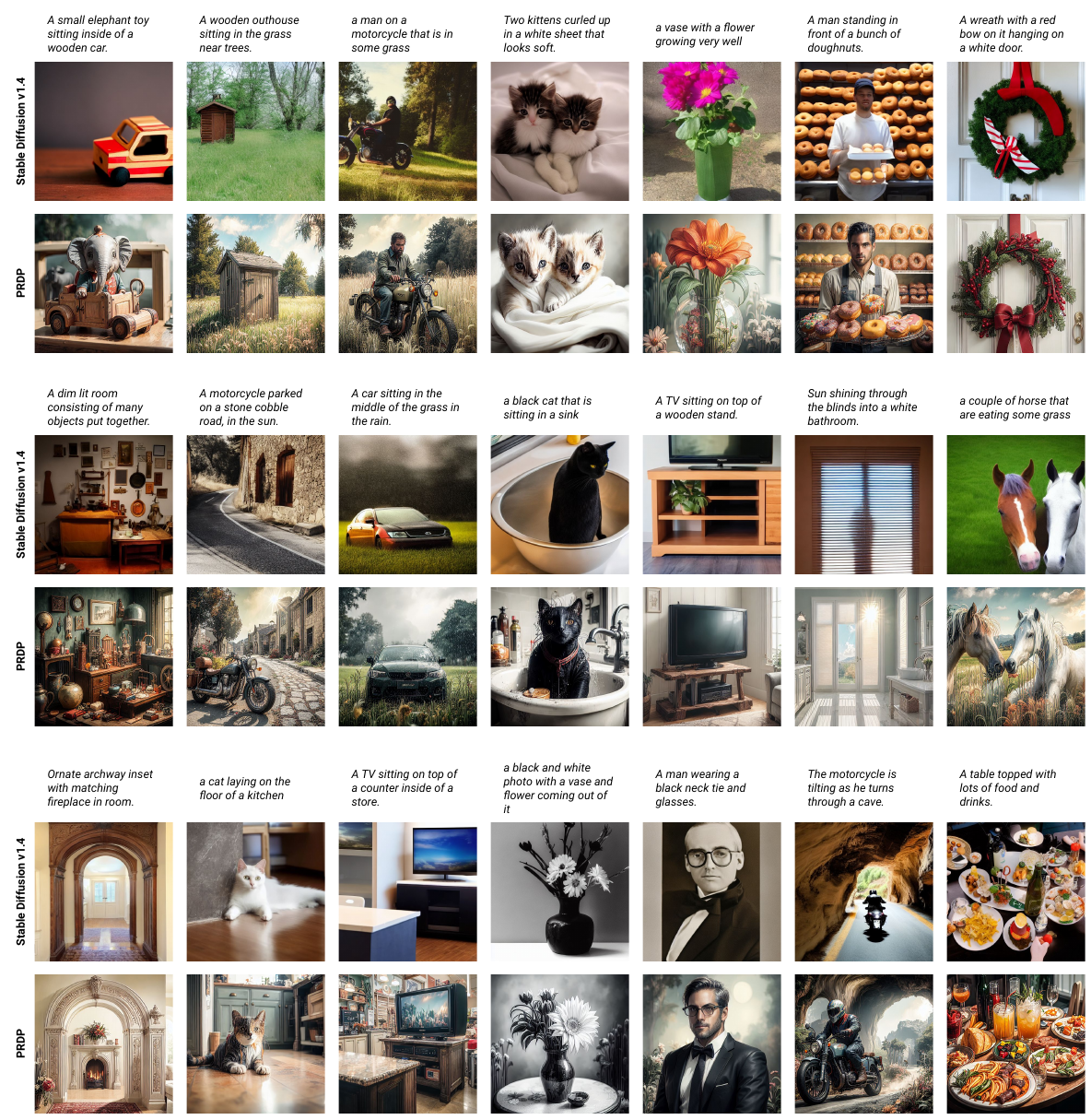}
    \caption{\textbf{Generation samples on unseen prompts from the HPD v2 photo benchmark.} PRDP is finetuned from Stable Diffusion v1.4 on the training set prompts of Pick-a-Pic v1 dataset, using a weighted combination of rewards: PickScore $= 10$, HPSv2 $= 2$, Aesthetic $= 0.05$. For each prompt, the generation sample from Stable Diffusion v1.4 and PRDP use the same random seed.}
    \label{fig:photo}
\end{figure}

\end{document}